%% file: main.tex
\documentclass[conference]{IEEEtran}
\IEEEoverridecommandlockouts
\usepackage{cite}
\usepackage{amsmath,amssymb,amsfonts}
\usepackage{algorithmic}
\usepackage{enumerate}
\usepackage[linesnumbered,ruled]{algorithm2e}
\usepackage{subfigure}

\usepackage{cmap}
\usepackage{graphicx}
\usepackage{textcomp}
\usepackage{xcolor}
\usepackage{bm}
\usepackage{marvosym}
\def\BibTeX{{\rm B\kern-.05em{\sc i\kern-.025em b}\kern-.08em
    T\kern-.1667em\lower.7ex\hbox{E}\kern-.125emX}}
    
\begin{document}

\newtheorem{definition}{Definition}
\newtheorem{lemma}{Lemma}
\newtheorem{proof}{Proof}

\title{DATA-WA: Demand-based Adaptive Task Assignment with Dynamic Worker Availability Windows }


\author{
\IEEEauthorblockN{Jinwen Chen}
\IEEEauthorblockA{\textit{University of Electronic Science } \\
\textit{and Technology of China}\\
Chengdu, China \\
202422081316@std.uestc.edu.cn}
\and
\IEEEauthorblockN{Jiannan Guo}
\IEEEauthorblockA{\textit{China Mobile (Suzhou)}\\
\textit{Software Technology Co., Ltd.} \\
Suzhou, China \\
guojiannan@cmss.chinamobile.com}
\and
\IEEEauthorblockN{Dazhuo Qiu}
\IEEEauthorblockA{
\textit{Department of Computer Science} \\
\textit{Aalborg University}\\
Aalborg, Denmark \\
dazhuoq@cs.aau.dk}\\
\and
\IEEEauthorblockN{Yawen Li}
\IEEEauthorblockA{
\textit{School of Economics and Management} \\
\textit{Beijing University of Posts and Telecommunications}\\
Beijing, China \\
warmly0716@bupt.edu.cn}\\
\and
\IEEEauthorblockN{Yan Zhao}
\IEEEauthorblockA{
\textit{Department of Computer Science} \\
\textit{Aalborg University}\\
Aalborg, Denmark \\
yanz@cs.aau.dk}
\and
\IEEEauthorblockN{Kai Zheng}
\IEEEauthorblockA{\textit{University of Electronic Science } \\
\textit{and Technology of China}\\
Chengdu, China \\
zhengkai@uestc.edu.cn}
}

\author{Jinwen~Chen\textsuperscript{1}, Jiannan~Guo\textsuperscript{2}, Dazhuo~Qiu\textsuperscript{3}, 
Yawen~Li\textsuperscript{4},
Guanhua~Ye\textsuperscript{4},
Yan~Zhao\textsuperscript{1,\Letter}, Kai~Zheng\textsuperscript{1,\Letter}
\\
\textsuperscript{1}University of Electronic Science and Technology of China, Chengdu, China\\
\textsuperscript{2}China Mobile (Suzhou)
Software Technology Co., Ltd.,
Suzhou, China\\
\textsuperscript{3}Aalborg University, Aalborg, Denmark\\
\textsuperscript{4}
Beijing University of Posts and Telecommunications,
Beijing, China\\
\thanks{\textsuperscript{\Letter} Corresponding authors: Yan Zhao and Kai Zheng. Yan Zhao is with Shenzhen Institute for Advanced Study, University of Electronic Science and Technology of China. Kai Zheng is with Yangtze Delta Region Institute (Quzhou), and School of Computer Science and Engineering, University of Electronic Science and Technology of China.}
jinwenc@std.uestc.edu.cn, guojiannan@cmss.chinamobile.com, dazhuoq@cs.aau.dk,\\ warmly0716@bupt.edu.cn, g.ye@bupt.edu.cn, yanz@cs.aau.dk, zhengkai@uestc.edu.cn\\
}

\maketitle

\begin{abstract}
With the rapid advancement of mobile networks and the widespread use of mobile devices, spatial crowdsourcing, which involves assigning location-based tasks to mobile workers, has gained significant attention. However, most existing research focuses on task assignment at the current moment, overlooking the fluctuating demand and supply between tasks and workers over time. To address this issue, we introduce an adaptive task assignment problem, which aims to maximize the number of assigned tasks by dynamically adjusting task assignments in response to changing demand and supply. We develop a spatial crowdsourcing framework, namely demand-based adaptive task assignment with dynamic worker availability windows, which consists of two components including task demand prediction and task assignment. In the first component, we construct a graph adjacency matrix representing the demand dependency relationships in different regions and employ a multivariate time series learning approach to predict future task demands. In the task assignment component, we adjust tasks to workers based on these predictions, worker availability windows, and the current task assignments, where each worker has an availability window that indicates the time periods they are available for task assignments. To reduce the search space of task assignments and be efficient, we propose a worker dependency separation approach based on graph partition and a task value function with reinforcement learning. Experiments on real data demonstrate that our proposals are both effective and efficient.

\end{abstract}


\input{introduction}

\input{problem_definition}
\input{task_prediction}
\input{task_assignment}

\input{experiment}
\input{related_work}

\input{conclusion}

\section{Acknowledgement}
This work is partially supported by NSFC (No. 62472068),  Shenzhen Municipal Science and Technology R\&D Funding Basic Research Program (JCYJ20210324133607021), and Municipal Government of Quzhou under Grant (No. 2023D044), and Key Laboratory of Data Intelligence and Cognitive Computing, Longhua District, Shenzhen.

\newpage
\bibliographystyle{IEEEtran}
\bibliography{ref}

\nocite{wu2020connecting}

\end{document}

%% file: introduction.tex
\section{Introduction}\label{sec:intro}
Along with the widespread  availability of GPS-enabled networked devices, e.g., smartphones,  Spatial Crowdsourcing (SC) has gained significant  attention in both academia and industry~\cite{tong2020spatial, zheng2022privacy, yao2023non, li2023acta, zhao2023preference, zhao2020predictive, Peng2023}. SC involves outsourcing location-based tasks (such as picking up passengers or delivering food and parcels) to individuals (i.e., crowd workers) through SC platforms like Didi and Uber Eats. This process is called \emph{task assignment}.

In a typical SC scenario, tasks must be assigned to crowd workers who are physically present in or near specific locations. The interaction between task requirements and worker availability is reflected in the demand and supply dynamics, describing how task availability and worker availability affect each other. When there are more tasks available than workers to complete them, there is high demand for workers, and vice versa. This interaction affects how tasks are assigned and completed in an SC platform. 
For instance, in ride-hailing services, a surge in passenger demand in a specific area results in high task demand, often causing a shortage of available drivers. To address this issue, platforms must utilize real-time data to increase the number of drivers nearby in the area. Conversely, during low-demand periods, the platform reallocates drivers to other areas, thereby reducing oversupply. This real-time feedback mechanism optimizes resource allocation and improves service efficiency. Similarly, in food delivery services, peak hours like lunch and dinner times cause a spike in orders and high task demand. During off-peak hours, fewer orders lead to an oversupply of drivers, with some going offline. This interaction shows how manage supply-demand imbalances in real time to ensure timely deliveries and operational efficiency.
Traditional task assignment methods~\cite{yao2023non, li2023acta, zhao2023preference,zhao2022Profit,zhao2021coalition,li2020group,li2020consensus} that focus on task assignment at the current moment without considering future demand dynamics often struggle to deal with the rapid and unpredictable changes in task demand and worker availability, leading to inefficiencies and suboptimal task assignments over the long term. 
It is important to accurately predict future task demands and consider worker availability to optimize current assignments. However, it is non-trivial to accurately predict the demands due to uncertain and dynamic spatio-temporal distributions, e.g., tasks might be dispersed over large and possibly uneven geographical areas.

Some studies consider task and worker predictions~\cite{9816080, Peng2023,zhao2020predictive,Wang2021Task,li2021preference}. For example, Wei et al.~\cite{9816080} propose a location-and-preference joint prediction model to predict workers' locations and preferences jointly at each time instance. Peng et al.~\cite{Peng2023} introduce a spatio-temporal prediction strategy that combines a gated recurrent unit and a variational autoencoder for crowdsourcing task prediction.
However, they ignore  the dynamics of task demands, which is crucial for accurately predicting the spatio-temporal distribution of tasks. 
The closest related research to ours is the work~\cite{zhao2020predictive}, which proposes a Prediction-based Task Assignment (PTA) approach that hybrids different learning models to predict the locations and routes of future workers and employs a graph embedding approach to estimate the distribution of future tasks. 
Nevertheless, it differs from our work in terms of the problem setting and assignment pattern. Specifically, PTA aims to maximize the number of assigned tasks by 
assigning a fixed task sequence to each worker, while our work focuses on maximizing the number of assigned tasks by assigning a dynamic task sequence to each worker based on the dynamics of task demands and the availability of workers. 
In SC, tasks and workers are continuously changing and moving, necessitating real-time updates and processing to ensure the optimal assignment.

This paper investigates the problem of Adaptive Task Assignment (ATA) in SC by focusing on demand dynamics and worker availability. Fig.~\ref{fig:example} illustrates a running example of the ATA problem with three workers denoted as $\{ w_1, w_2, w_3\}$, and nine tasks denoted as $\{s_1, \dots ,s_9 \}$.
Each worker can only perform tasks within a reachable distance of 1.2 units. 
In addition, each spatial task, published and expired at different time instances, is labeled with its location where it will be performed only once.
The straightforward approach, known as Fixed Task Assignment (FTA) algorithm, is to assign each worker a fixed sequence of tasks to be completed in order while satisfying spatial-temporal constraints.
In our example, we assign the task sequence $(s_1, s_3)$ to $w_1$ and $(s_2, s_4)$ to $w_2$, achieving the maximal number of assigned tasks at time instance 1. Similarly, in the time instance 4, we assign task $s_7$ to $w_3$.
However, the remaining tasks are rejected because no worker can reach the task locations after completing their assigned task sequences. Consequently, the total number of assigned tasks is 2 + 2 + 1 = 5.

\begin{figure}[htbp]
\vspace{-0.4cm}
    \centering
    {\includegraphics[width = 0.5\textwidth ]{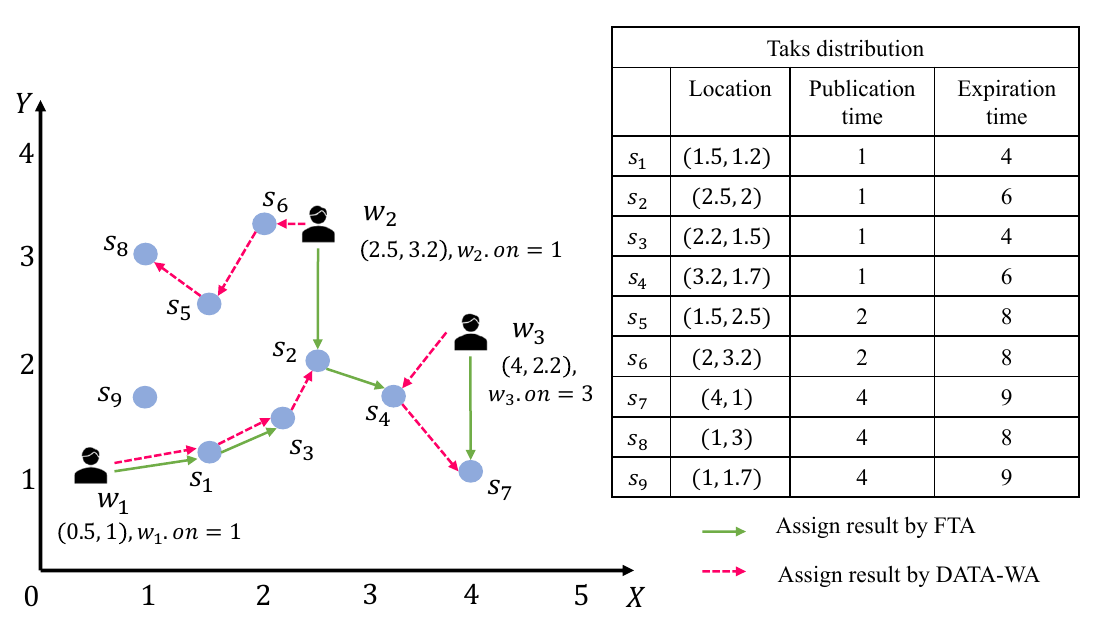}}
    \vspace{-0.7cm}
    \caption{Running Example}
    \label{fig:example}
    \vspace{-0.1cm}
\end{figure}

We show that the ATA problem is NP-hard (see Lemma~\ref{lem:np}). 
To solve ATA, we propose an SC framework, namely \textbf{\underline D}emand-based \textbf{\underline A}daptive \textbf{\underline T}ask \textbf{\underline A}ssignment with dynamic \textbf{\underline W}orker \textbf{\underline A}vailability windows (DATA-WA), which adjusts task assignment based on real-time and predicted task demand dynamics as well as dynamic worker availability. For task demand prediction, we employ a multivariate time series learning approach, Dynamic Dependency-based Graph Neural Network, to predict  future task demands across different regions. To the best of our knowledge, our approach is the first to consider the dependency relationships between task demands in different regions.
For task assignment, we use a Worker Dependency Separation approach based on a graph partition technique. This approach segregates workers into independent clusters arranged in a tree structure based on their locations and availability windows, where worker availability windows refer to the specific time periods during which workers are available to perform tasks. These windows can vary in duration and may include specific start and end times. 
In the constructed tree structure,
workers in sibling nodes are independent, thus reducing the search space. During the search process, we utilize a trained Task Value Function to select the optimal task sequence for workers and adaptively adjust current task assignments, minimizing the need for multiple backtracking processes.
The DATA-WA framework can handle large volumes of data efficiently, especially in urban areas with high densities of tasks. Fig.~\ref{fig:example} illustrates the task assignment results by applying DATA-WA, which assigns eight tasks.

Our contributions can be summarized as follows:

1) We identify and study in depth an Adaptive Task Assignment (ATA) problem, considering task demand dynamics and worker availability in the context of SC.

2) We design a multivariate time series learning method, called Dynamic Dependency-based Graph Neural Network, to capture demand dependencies among different regions on future task demand prediction.

3) We design a demand-based adaptive task assignment method considering dynamic worker availability windows to assign tasks.

4) Experimental results demonstrate that our proposed approaches are both effective and efficient when applied to real datasets.

The remainder of this paper is organized as follows. The preliminary concepts and framework overview are introduced in Section~\ref{II}. We then present the task demand prediction and task assignment methodology in Section~\ref{III} and Section~\ref{IV}, respectively, followed by the experimental results in Section~\ref{V}. Section~\ref{VI} surveys the related work, and Section~\ref{VII} concludes this paper.

%% file: problem_definition.tex
\section{Problem Statement}
\label{II}
We proceed to present necessary preliminaries, define the problem addressed and then give an overview of our framework. Table~\ref{tab:notation} lists the major notations used throughout the paper.

\begin{table}[htbp]
    \centering
    \caption{SUMMARY OF NOTATIONS}
    \vspace{-0.2cm}
    \begin{tabular}{c|c}
    \hline
         Notation& Definition \\ \hline
         $s$&Spatial task\\ \hline
         $s.l$& Location of spatial task $s$\\ \hline
         $s.p$& Publication time of spatial task $s$ \\ \hline
         $s.e$& Expiration time of spatial task $s$\\ \hline
         $w$& Available worker \\ \hline
         $w.l$& Location of worker $w$\\ \hline
         $w.d$& Reachable distance of worker $w$\\ \hline
         $w.\mathit{on}$& Online time of worker $w$\\ \hline
         $w.\mathit{off}$& Departure time of worker $w$\\ \hline
         $S$&A task set \\ \hline
         $R(S)$&A task sequence on tasks in $S$ \\ \hline
         $\mathit{VR}(w)$&A valid scheduled task sequence of worker $w$\\ \hline
         $T_w$& Availability window of worker $w$\\ \hline
         $\Delta T$ & Time interval \\ \hline
         $t(l)$& Arrival time of location $l$\\ \hline
         $c(a,b)$& Travel time from $a$ to $b$ \\ \hline        
         $td(a,b)$& Travel distance from $a$ to $b$ \\ \hline
         $A$& A spatial task assignment\\ \hline
         $\mathbb{A}$& A spatial task assignment set\\ \hline
          
    \end{tabular}
    
    \label{tab:notation}
    \vspace{-0.5cm}
\end{table}

\subsection{Preliminary}

\begin{definition}[Task]
    A Task, denoted by $s = (l,p,e)$, has a location $s.l$, a publication time $s.p$, and
    an expiration time $s.e$. 
\end{definition}

\begin{definition}[Worker]
A worker can be in an either online or offline 
mode. A worker is considered offline when unable to perform tasks, and online when ready to accept tasks.
An online worker, denoted by $w = (l,d,\mathit{on},\mathit{off})$, consists of a location $w.l$, a reachable distance $w.d$, an online time $w.\mathit{on}$ and an offline time $w.\mathit{off}$.
\end{definition}

\begin{definition}[Task Sequence]
Given an online worker $w$ and a set of assigned tasks $S_w$, 
a task sequence on $S_w$, denoted as $R(S_w)$, represents the order in which $w$ performs the tasks in $S_w$. The arrival time of $w$ at the location of task $s_i\in S_w$ can be calculated in Eq.~\ref{eq:arrTime}.
    \begin{equation}\label{eq:arrTime}
t_{R,w}\left(s_i.l\right)= \begin{cases}t_{R,w}\left(s_{i-1}.l\right)+c\left(s_{i-1}.l, s_i.l\right) & i > 1 \\ t_{\mathit{now }}+c\left(w.l, s_i.l\right) & i=1,\end{cases}
\end{equation}
    where $t_{R,w}\left(s_i.l\right)$ denoted the arrival time of $w$ at $s_i.l$, $c(a, b)$ denotes the travel time from location $a$ to location $b$, $t_{now}$ is the current time, and $w.l$ denotes the current location of $w$, from which $w$ begins to accept the task assignment.
\end{definition}

\begin{definition} [Valid Task Sequence]
\label{VR}
A task sequence $R(S_w)$ is called a valid task sequence for a worker $w$, denoted as $\mathit{VR}(S_w)$, if it satisfies the following constraints: 

\begin{enumerate}[i.]
        \item all the tasks in this sequence can be completed before their expiration times, i.e., $t_{R,w}(s_i.l) < s_i.e$, and 
        \item all the tasks in this sequence can be completed before the offtime of worker $w$, i.e., $t_{R,w}(s_i.l) < w.\mathit{off}$, and
        \item all the tasks in this sequence are located in the reachable range of $w$, i.e., $\mathit{td}(w.l, s_i.l) < w.d$, where $\mathit{td}(a, b)$ is the travel  distance between location $a$ and $b$. 
    \end{enumerate}
\end{definition}

\begin{definition}[Spatial Task Assignment]Given a set of workers $W$ and a set of tasks $S$, a spatial task assignment, denoted by $A$, consists of a set of $(w, \mathit{VR}(S_w))$ pairs.
\end{definition}

In this work, we follow the single task assignment mode~\cite{kazemi2012geocrowd}, where each task can  be completed
by only one worker, and assume that each worker can perform at most one task at a time, which is practical.  
Let $A.S$ denote the set of tasks that are assigned to all workers, i.e., $A.S = \cup_{w \in W}\mathit{VR}(S_w)$, and $\mathbb{A}$ denote all possible ways of assignments. Our problem investigated in our paper can be formally stated as follows.

\textbf{Problem Statement.} Given a set of tasks $S$ and a set of workers $W$, 
our Adaptive Task Assignment (ATA) problem aims to find the global optimal assignment $A_{opt}$, such that $\forall A_i \in \mathbb{A}, |A_i.S|\le |A_{opt}.S|$

\begin{lemma}
\label{lem:np}
The ATA problem is NP-hard.
\end{lemma}

\begin{proof}
\label{prf:np}
The hardness of the ATA problem can be proved by constructing a polynomial time reduction from the Maximum Coverage (MC) problem and showing that an solution to ATA can be reduced to a solution to MC. 

In the MC problem,
we are given a collection of sets $\mathbb{B}=$ $\left\{B_1, B_2, \ldots, B_n\right\}$ over a set of objects $O$, where $B_i \subseteq O$, and a positive integer $k$. The MC problem is to find a subset $\mathbb{B}^{\prime} \subseteq \mathbb{B}$ such that $\left|\mathbb{B}^{\prime}\right| \leq k$ and the number of covered elements in $\mathbb{B}^{\prime}$ (i.e., $\left.\left|\cup_{B_i \in \mathbb{B}^{\prime}} B_i\right|\right)$ is maximized.

Consider the following instance of our ATA problem.
We are given a collection of task sets $\mathbb{VR} = \left\{\mathit{VR}_1, \mathit{VR}_2, \ldots\mathit{VR}_n\right\}$ over a set of tasks $S$, where $\mathit{VR}_i \subseteq S$. We are also give a set of $k$ workers, $W=\left\{w_1, w_2, \ldots, w_k\right\}$. Our ATA problem is to find a subset $\mathbb{VR}^{\prime} \subseteq \mathbb{VR}$ such that $\left|\mathbb{VR}^{\prime}\right| \leq k$ and the number of assigned tasks, i.e., $\left|\cup_{\mathit{VR_i} \in \mathbb{VR}^{\prime}} \mathit{VR_i}\right| $, is maximized.

Given an instance of the MC problem with a collection of sets $\mathbb{B}=$ $\left\{B_1, B_2, \ldots, B_n\right\}$ and 
a positive integer $k$, we construct an instance of the ATA problem, where $\mathbb{VR}$, $\mathit{VR}_i$, $S$,  and $\mathbb{VR}^{\prime}$ correspond to $\mathbb{B}$, $\mathit{B}_i$, $O$,  and $\mathbb{B}^{\prime}$ in the original MC instance, respectively.
Therefore, the instance of the ATA problem can be reduced from the instance of MC.

A solution to the ATA problem would select a subset of workers $W^{\prime}\in W$ ($|W^{\prime}|\le |W|$) and assign them tasks $\mathbb{VR}^{\prime} = \left\{\mathit{VR}^{w_1}, \mathit{VR}^{w_2}, \ldots, \mathit{VR}^{w_{|W^{\prime}|}}\right\}$, where the task assignment maximizes the total assigned tasks.
This is directly equivalent to finding a subset  $\mathbb{B}^{\prime}\subseteq \mathbb{B}$, with $\left|\mathbb{B}^{\prime}\right| \leq k$, that maximizes the number of covered elements.

If we can solve the ATA problem instance efficiently (i.e.,
in polynominal time), we can solve a MC problem
by transforming it to the corresponding ATA problem instance
and then solve it efficiently. This contradicts the fact that the MC problem is NP-hard~\cite{mc}, and so there cannot be an efficient solution to the ATA problem instance that is then NP-hard. Since the ATA problem instance is NP-hard, the ATA problem is also NP-hard.

\end{proof}

\begin{figure*}[htbp]
\centerline
{\includegraphics[width = 0.95\textwidth ]{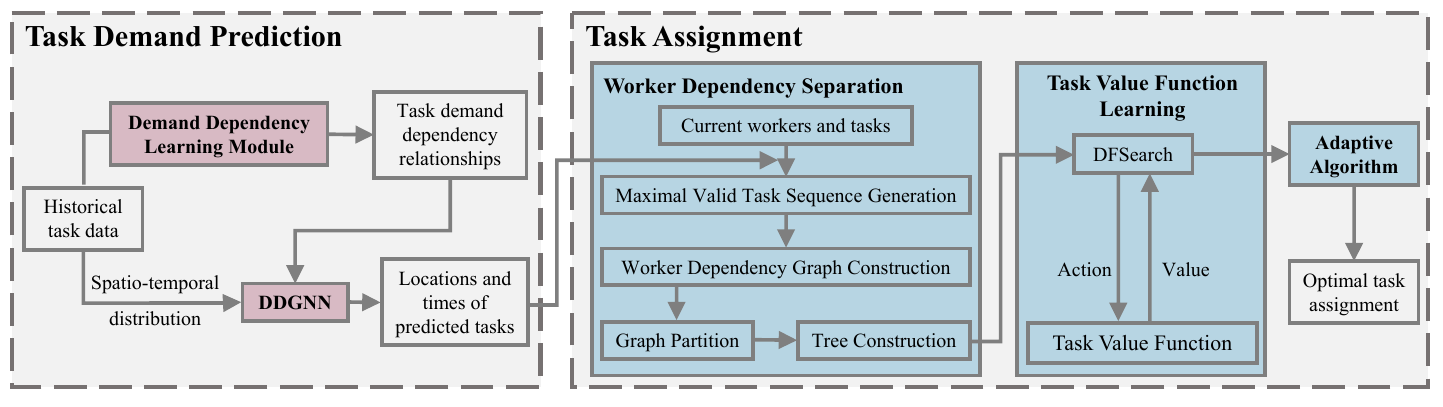}}
\vspace{-0.5cm} 
\caption{Framework Overview}
\label{Framework}
\vspace{-0.5cm}
\end{figure*}

\subsection{Framework Overview}
We propose a framework, namely \textbf{\underline D}emand-based \textbf{\underline A}daptive \textbf{\underline T}ask \textbf{\underline A}ssignment with dynamic \textbf{\underline W}orker \textbf{\underline A}vailability windows (DATA-WA), to adaptively assign tasks to workers based on demand dynamics.
We first give an overview of the framework and then provide specifics on each component in the following sections.

The innovation of our proposed framework 
lies in its consideration of the future task's demand dynamics. 
It assigns a dynamic task sequence to each worker, updating and processing this sequence in real-time to ensure optimal assignment.
The main challenges involve accurately predicting the spatio-temporal demand of future tasks and strategically planning assignments for both current and predicted tasks for current workers.
To address these challenges, we introduce a novel SC framework with two key components: task demand prediction and task assignment, as illustrated in Fig.~\ref{Framework}.

The first component focuses on predicting future task demands using historical data. We consider the dependency relationships of task demands across different regions and employ a multivariate time series learning approach to forecast future demands.
To model these dependency relationships, we use a demand dependency learning module to estimate a graph adjacency matrix based on historical data. We then develop a \textbf{\underline D}ynamic \textbf{\underline D}ependency-based \textbf{\underline G}raph \textbf{\underline N}eural \textbf{\underline N}etwork (DDGNN) to capture both spatial and temporal dependencies, aiding in predicting the locations and times of tasks.

The second component involves planning a suitable task sequence for each worker in real-time to achieve global optimal task assignments.
First, in the Worker Dependency Separation phase, we compute the Maximal Valid Task Sequence for all workers and construct a worker dependency graph. We then apply a graph partitioning technique to divide workers into independent clusters, 
which are then organized into a tree structure.
Next, we perform a depth-first search algorithm to traverse the tree and assign the optimal task sequences for each worker. This data is then used to train a Task Value Function (TVF). 
Finally, during the adaptive algorithm phase, we use the trained TVF to select the optimal task sequence for all workers and adaptively adjust their task assignments based on changes in supply and demand. This approach minimizes the need for multiple backtracking processes, ensuring a global optimal task assignment. 

%% file: task_prediction.tex
\section{Task Demand Prediction}
\label{III}
Predicting future dynamic task demands is crucial for dynamically adjusting the task sequence for workers to achieve optimal task assignments. We approach task prediction using a grid-based method by partitioning the study area into disjoint and uniform grids. Each grid represents a specific type of area, such as schools, shopping malls, or food streets, thereby reflecting real-world scenarios.
The historical task data from multiple grid cells can be treated as a multivariate time series, which will be detailed in the following preliminaries. Predicting this multivariate time series is essentially equivalent to addressing the task demand prediction problem. To address this, we propose a two-module method to predict future task demands across different regions. 
As shown in Fig.~\ref{Framework}, the proposed method consists of a Demand Dependency Learning Module and a Dynamic Dependency-based Graph Neural Network (DDGNN), which work together to predict task demands based on the spatio-temporal distribution of tasks. 
In the following sections, we first give some preliminaries and then detail the two modules. 

\subsection{Preliminaries}

\textbf{Task Multivariate Time Series.} Unlike models such as LSTM~\cite{lstm} that only consider one variable, multivariate time series incorporates multiple variables at each time step~\cite{cui2021metro, deng2025million, chen2021daemon}. Deriving from the multivariate approach, we propose the Task Multivariate Time Series for multiple grids. In each grid cell $i\in \{1,\ldots,M\}$, we define a multivariate time series: $\boldsymbol{C}_i = \langle \boldsymbol{c}^{t_0}_i, \boldsymbol{c}^{t_0+k\Delta T}_i,\ldots,\boldsymbol{c}^{t_0+(P-1)k\Delta T}_i\rangle$ (we use bold letters $\boldsymbol{C}$ and $\boldsymbol{c}$ to denotes vectors). It consists of a sequence of vectors $\boldsymbol{c}$ in increasing time order (starting from time $t_0$), with each vector having $k$ (user-specified, and $k>1$) dimensions. 
Each dimension corresponds to the task occurrence in a specific time interval $\Delta T$, and a binary value (1 for yes, 0 for no) is assigned to each dimension based on whether tasks occur during this time interval. 
For example, $\boldsymbol{c}^{t_0+(P-1)k\Delta T}_i$ has $k$ dimensions, which shows the task occurrence during the time interval from $t_0+(P-1)k\Delta T$ to $t_0+Pk\Delta T$, with each dimension denoting the task occurrence in $\Delta T$.
Given a task set $S$ that occur in cell $i$ during time interval $[t_0, t_0+Pk\Delta T]$ ($t_0+Pk\Delta T$ is the right time boundary of vector $\boldsymbol{c}_i^{t_0+(P-1)k\Delta T}$), the binary value of $j$ dimension $(j \in {1,\ldots,k})$ of vector $\boldsymbol{c}_i^{t}[j]$ $(t \in \{t_0, t_0+k\Delta T, \ldots, t_0+(P-1)k\Delta T\})$ is defined in Eq.~\ref{eq:Grid1}. 
\begin{equation}\label{eq:Grid1}
\boldsymbol{c}_i^{t}[j]=
\begin{cases}
1   & \exists s \in S;  t+(j-1)\Delta T \leq s.p < t+j\Delta T \\
0   & otherwise\\
\end{cases}
\end{equation}
Then for each vector in the task multivariate time series, it covers $k\Delta T$ time intervals. Based on $P$ historical record vectors in $\boldsymbol{C}_i$, our goal is to predict the next vector during time interval from $t_0+Pk\Delta T$ to $t_0+(P+1)k\Delta T$, i.e., $\boldsymbol{c}_i^{t_0+Pk\Delta T}$. 

For example, as shown in Fig.~\ref{exampleC}, we set $k=3$, which indicates that we consider three time intervals $3\Delta T$ for each vector $\boldsymbol{c}$. 
Specifically, between $t_0$ and $t_0 + \Delta T$, there are tasks published in grid $i$, 
resulting in the first dimension of the first vector being $1$. Similarly, for the second dimension in the first vector, tasks are published between $t_0 + \Delta T$ and $t_0 + 2\Delta T$, 
so this value is also $1$.
However, in the third dimension of the first vector, no tasks are published, yielding a value of $0$. Therefore, the time series is represented as a vector $\boldsymbol{c} = \langle 1, 1, 0\rangle$, reflecting the presence of tasks across these three time intervals.

\begin{figure}[htbp]
\vspace{-0.4cm}
\centerline
{\includegraphics[width = 0.48\textwidth ]{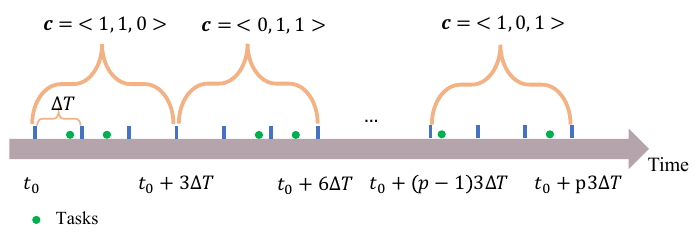}}
\vspace{-0.3cm}
\caption{Example of Task Multivariate Time Series ($k=3$)}
\label{exampleC}
\vspace{-0.3cm}
\end{figure}

\textbf{Grid Graph.}
    A grid graph is denoted as $G=(V, E)$, where $V$ is a set of nodes representing grid cells, and $E$ is a set of edges representing the task demand dependencies between these cells. An edge $e_{ij} = (v_i, v_j)$ exists if the task demands in grid cells $v_i$ and $v_j$ affect each other. Specifically, if an increase in task demands in one grid cell (e.g., a region in a city) leads to a change in task demands in another grid cell after some time, it indicates a strong dependency between these cells, and they are connected in the grid graph. For simplicity, we refer to grid graph as graph when the context is clear. 

\textbf{Graph Adjacency Matrix.}
    The adjacency Matrix at time instance $t$ is denoted by $\mathcal{A}^t \in \textbf{R}^{M \times M}$,
    where $M$ is the number of nodes (grid cells) in graph $G$. For any two nodes $v_i$ and $v_j$,  $\mathcal{A}^t_{ij}=1$ if $(v_i, v_j) \in E$; $\mathcal{A}^t_{ij}=0$ if $(v_i, v_j) \notin E$.
 




\textbf{Dilated Causal Convolution.}
Dilated Causal Convolution~\cite{wu2020connecting} effectively captures a node’s temporal trends by allowing an exponentially large receptive field as layer depth increases. It handles long-range sequences efficiently in a non-recursive manner. 
Specifically, given a sequence input $x = \{x_1,x_2,\ldots,x_J\}$ and a filter function $f(\cdot) \in \textbf{R}^K$ ($K$ is the dimension of filter that is set to $3$ in our work), the output of dilated causal convolution operation of $x$ with $f(\cdot)$ at step $j$ is represented as follows:
\begin{equation}
y_j = \sum_{i=0}^{K-1}{f(i)\cdot x_{j-i \cdot d}},
\end{equation}
where $d$ is the dilation factor that controls the skipping distance.



\subsection{Demand Dependency Learning Module.}
Tasks across different regions are interconnected, meaning that changes in task demands in one area can affect those in others. For instance, in a city, when university classes end, students often head to a nearby restaurant district, causing an initial surge in ride requests from the university. Later, when the students finish dining and socializing, they request rides home, leading to increased demand in the restaurant district. This example illustrates how rising demand in one region can subsequently impact another.

Graph-based methods for capturing relationships between different observations have become prevalent in multivariate time series prediction \cite{wu2020connecting, wu2019graph}. In this paper, we propose a dynamic time-based adjacency matrix, $\mathcal{A}^t$, to represent the dependencies of task demands across different regions at each time instance $t$. We develop a demand dependency learning module that learns the graph adjacency matrix through end-to-end learning using stochastic gradient descent.


First, we initialize two node embedding features (i.e., $\mathbb{M}_1$ and $\mathbb{M}_2$) with neural networks from historical task data $\mathbb{C}^t = \{\boldsymbol{c}_1^t, \boldsymbol{c}_{2}^t,..., \boldsymbol{c}_M^t\}$ at time instance $t$, where $\boldsymbol{c}_i^t$ denotes the encoding feature of $i^{th}$ cell at time instance $t$ and $M$ is the number of grid cells, as illustrated below:
\begin{equation}
\label{eq:M1}
\mathbb{M}_1 = F_{\theta_1}(\mathbb{C}^{t})
\end{equation}
\begin{equation}
\label{eq:M2}
\mathbb{M}_2 = F_{\theta_2}(\mathbb{C}^{t}),
\end{equation}
where $F_{\theta}$ is a neural network (e.g., fully connected layer) with parameters $\theta$. 

Second, we designate $\mathbb{M}_1$ as the source node embedding and $\mathbb{M}_2$ as the target node embedding. By multiplying $\mathbb{M}_1$ and $\mathbb{M}_2$, we calculate the spatial dependency weights between the source nodes and the target nodes in Eq.~\ref{eq:A}.
\begin{equation}
\label{eq:A}
\mathcal{A}^{t} = SoftMax(tanh(\mathbb{M}_1\mathbb{M}_2^T+\mathbb{M}_2\mathbb{M}_1^T)),
\end{equation}
where the tanh activation function maps the results to a range between $-1$ and $1$, and the SoftMax activation function is applied to normalize the adjacency matrix.
\begin{figure}[htbp]
\vspace{-0.5cm} 
\centerline
{\includegraphics[width = 0.35\textwidth ]{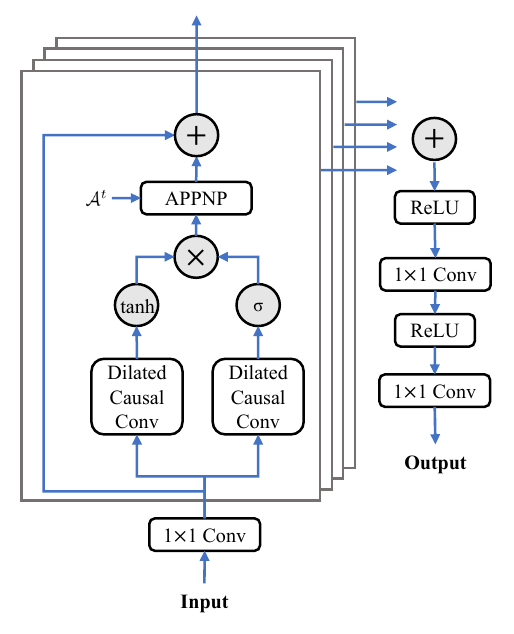}}
\vspace{-0.5cm} 
\caption{DDGNN Overview}
\label{fig2}
\vspace{-0.5cm} 
\end{figure}

\subsection{Dynamic Dependency Graph Neural Network}

Dilated causal convolution has proven effective in capturing temporal trends~\cite{wu2020connecting, quan2023detection}.
In our proposed Dynamic Dependency-based Graph Neural Network (DDGNN) model, we employ dilated causal convolution to identify the temporal dependencies. Then, we utilize a graph propagation approach to integrate a node's information with that of its neighbors, thereby addressing task demand dependencies within the graph, as illustrated in Fig~\ref{fig2}.

Gating mechanisms are crucial in multivariate time series forecasting as they regulate the information flow within the dilated causal convolution network. We utilize dilated causal convolution to identify the temporal dependencies in cell $i$, as represented by Eq.~\ref{eq:gating}.
\begin{equation} 
\label{eq:gating}
Z_i = tanh(\theta_1\boldsymbol{C}_i, b_1) \odot \sigma(\theta_2\boldsymbol{C}_i, b_2),
\end{equation}
where $\theta_1$, $\theta_2$, $b_1$, and $b_2$ are dilated causal convolution parameters, $\odot$ is the element-wise product, $tanh( \cdot )$ is the activation function, and $\sigma(\cdot)$ is the sigmoid function that determines the proportion of information flow to the next layer.

Next, we use Approximate Personalized Propagation of Neural Predictions (APPNP)~\cite{appnp} as the propagation layer to extract a node's feature by aggregating and transforming its neighborhood information. The propagation step is defined as follows:
\begin{equation}
\mathbb{Z}^{(h+1)}_t = \alpha \mathbb{Z}^{(0)}_t + (1-\alpha)\hat{\mathcal{A}}^t\mathbb{Z}^{(h)}_t
\end{equation}
\begin{equation}
\mathbb{Z}^{(H)}_t = ReLU(\alpha \mathbb{Z}^{(0)}_t + (1-\alpha)\hat{\mathcal{A}}^t\mathbb{Z}^{(h-1)}_t),
\end{equation}
where $\mathbb{Z}^{(0)}_t \in \textbf{R}^{M \times k}$ represents the input hidden states at time instance $t$ output by the previous layer, and $\mathbb{Z}^{(H)}_t$ are the output feature by propagation. The $\hat{\mathcal{A}}^t = \hat{D}^{-1/2}(\mathcal{A}^t+I)\hat{D}^{-1/2}$ is the normalized adjacency matrix, where $\hat{D}$ is the diagonal degree matrix $\hat{D}_{ii} = 1 + \sum_j{\mathcal{A}^t_{ij}}$, $\alpha$ is a hyper-parameter for controlling the restart probability, and $H$ defines the number of power iteration steps and $h \in [0, H - 2]$.

After applying the DDGNN approach, we obtain the vector $\boldsymbol{c}_i^{t_0+pk\Delta T}$ for the $i^{th}$ cell.
If $\boldsymbol{c}^{t_0+pk\Delta T}_i[j]$ exceeds a given threshold (i.e, 0.85 in our experiments), we predict that a task will be published in the $i^{th}$ cell during the time interval from $t_0+pk\Delta T$ to $t_0+(p+1)k\Delta T$. Both predicted and current tasks are considered in the next task assignment process.

%% file: task_assignment.tex
\section{Task Assignment}
\label{IV}
In this section, we introduce a task assignment component designed to address the ATA problem by planning a dynamically valid task sequence for all available workers. This component considers current and predicted tasks alongside worker availability windows to achieve optimal task assignments. Each worker's availability window reflects the time periods they are available for task assignments and can change dynamically due to factors such as breaks, shifts, or unforeseen circumstances.

We first present the Worker Dependency Separation (WDS) approach, which uses tree-based graph partitioning to divide workers into independent clusters and construct a hierarchical tree structure. A depth-first search algorithm is then employed to gather optimal results. Additionally, we incorporate a reinforcement learning approach to train a Task Value Function (TVF) based on these optimal results. Finally, we propose an adaptive algorithm for task assignment that adjusts the allocation of tasks to workers in response to changes in demand and supply, ensuring the most efficient task distribution.

\input{worker_dependency_separation}

\subsection{Task Value Function Learning}
Recursive search can often lead to significant computational complexity and instability. In Reinforcement Learning (RL), the value function evaluates the long-term value of a state-action pair, which is essential for guiding the agent in choosing the optimal action. To improve task assignment efficiency, we use a depth-first search method to gather training data and train a Task Value Function (TVF) based on this data.

In RL, the value function evaluates the long-term value of a state-action pair. It consists of three parts: state, action, and reward. In task assignment, the ``state" consists of the states of all remaining workers and tasks (i.e., locations and publication times).
The ``action" involves selecting a specific worker and determining a task sequence for them. 
The ``reward" represents the assigned number of tasks reflecting the effectiveness of the task assignment process. The state-action value function $\mathit{TVF}(st, at)$ represents the expected value of the cumulative reward that the agent may obtain in the future after performing action $at$ in state $st$:
\begin{equation}
\mathit{TVF}(st, at)=\mathbb{E}\left[V \mid ST=st, AT=at\right],
\end{equation}
where $V$ denotes the cumulative reward, and the expectation is taken over the distribution of possible future states and rewards.

We explain the depth-first search method, $\mathit{DFSearch}$, in Algorithm~\ref{dfs}. The algorithm starts by computing the remaining available workers $W_C$ of all nodes excluding $W_N$ in the tree (lines~\ref{dfs2}--\ref{dfs4}). If there are still workers to be probed, we will sequentially examine each available worker in $W_N$ (lines~\ref{dfs5}--\ref{dfs6}). We use depth-first-search to determine the optimal assignment number (line~\ref{dfs8}) and define the current state as $(W_N + W_C, S)$ and action as $(w, q)$, which corresponds to a RL process that appends current state, action, and reward (i.e., $st$, $at$, and $opt$) to training data $U$ (lines~\ref{dfs9}--\ref{dfs11}).
On the other hand, if all the workers have been considered, the algorithm will initiate the $\mathit{DFSearch}$ process on each child node of $N$ (lines~\ref{dfs15}--\ref{dfs16}).
\vspace{-0.2cm}
\begin{algorithm}
\small
    \label{dfs}
    \SetAlgoLined 
    \caption{DFSearch}
    \SetKwInOut{Input}{Input}
    \SetKwInOut{Output}{Output}
    \Input{A node $N$, a set of tasks $S$, A set of workers $W_{N}$} 
    \Output{$opt$} 
    $opt \gets 0; W_C\gets \emptyset; U \gets \emptyset;$  \label{dfs1}\\
     \For{each child node $N_i$ of N}{ \label{dfs2}
        $W_C + \gets W_{N_i}$; \label{dfs3}
        } \label{dfs4}
    \eIf{$W_N \ne \emptyset$}{ \label{dfs5}
        \For{each worker $ w \in W_N$}{ \label{dfs6}
            \For{each sequence $ q \in Q_w$}{ \label{dfs7}
            $opt \gets max\{\mathit{DFSearch}(N, S-q, W_N - w)+|q|, opt\}$; \label{dfs8}\\
            $st \gets (W_N + W_C, S)$; \label{dfs9}\\
            $at \gets(w, q)$; \label{dfs10}\\
            $U + \gets (st, at, opt)$; \label{dfs11}\\
            } \label{dfs12}
        } \label{dfs13}
    }{ \label{dfs14}
        \For{each child node $N_i$ of N}{ \label{dfs15}
        $opt + \gets \mathit{DFSearch}(N_i, S, W_{N_i})$; \label{dfs16}
        } \label{dfs17}
    } \label{dfs18}
    \KwRet $opt$; \label{dfs19}
\end{algorithm}
\vspace{-0.5cm}

During learning, we apply Q-learning~\cite{watkins1992q} updates on samples (or mini-batches) of experience $\left(st, at, opt\right)$, drawn uniformly at random from the stored samples generated by $\mathit{DFSearch}$. The Q-learning update uses the following loss function:
\begin{equation}
L\left(\theta\right)=\mathbb{E}_{\left(st, at, opt\right) \sim U}\left[\left(opt-\mathit{TVF}\left(st, at ; \theta\right)\right)^2\right],
\end{equation}
where $\theta$ are the parameters of the TVF.

The memory consumption for storing all state-action pairs in Algorithm~\ref{dfs} is $O(|U| \times (|W| + |S| + |RS|))$, where $|U|$ represents the size of the training data $U$, and $|RS|$ denotes the average number of reachable tasks for each worker.

We then use the trained TVF to solve the assignment problem, referred to as $\mathit{DFSearch\_TVF}$, as outlined in Algorithm~\ref{rl}. The inputs include the root node $N$ of the sub-tree to be traversed, the remaining unassigned task set $S$, and the remaining available workers $W_N$ at node $N$. The output is the assignment result $\mathit{SA}$ within the sub-tree rooted at node $N$.

The algorithm begins by computing the remaining available workers, $W_C$, for all nodes in the tree, excluding $W_N$ (lines~\ref{rl2}--\ref{rl4}). If there are still workers to be probed, we will examine the first available worker in $W_N$ (lines~\ref{rl5}--\ref{rl6}). We define the current state as $\{W_N + W_C, S\}$ and the action as $\{w, q\}$, corresponding to reinforcement learning (RL). The value function TVF is then used to calculate the optimal task sequence $q_\mathit{best}$, which maximizes the reward from the valid task sequence set $Q_w$ (lines~\ref{rl7}--\ref{rl8}). We then append the best assigned result to output and recursively call the $\mathit{DFSearch\_TVF}$ approach by passing in the updated remaining task set $S-q_\mathit{best}$ and worker set $W_N-w$ (lines~\ref{rl9}--\ref{rl10}). Conversely, if all the workers have been counted, the algorithm will initiate the $\mathit{DFSearch\_TVF}$ process on each child node of $N$ (lines~\ref{rl12}--\ref{rl14}).

\vspace{-0.2cm}
\begin{algorithm}
\small
    \label{rl}
    \SetAlgoLined 
    \caption{DFSearch\_TVF}
    \SetKwInOut{Input}{Input}
    \SetKwInOut{Output}{Output}

    \Input{A node $N$, a set of tasks $S$, A set of workers $W_{N}$}
    \Output{$\mathit{SA}$}
    $\mathit{SA} \gets \emptyset; W_C \gets \emptyset;$ \label{rl1}\\
     \For{each child node $N_i$ of N}{ \label{rl2}
        $W_C + \gets W_{N_i}$; \label{rl3}
        } \label{rl4}
    \eIf{$W_N \ne \emptyset$}{ \label{rl5}
        $w \gets W_N[0]$;  \label{rl6}\\
        $st \gets (W_N + W_C, S);$ \label{rl7}\\
        $q_{best} \gets argmax_{q \in Q_W}\mathit{TVF}(st, (w, q));$ \label{rl8}\\
        $\mathit{SA} + \gets \{w, q_{best}\};$ \label{rl9}\\
        $\mathit{SA} + \gets \mathit{DFSearch\_TVF}(N, S - q_{best},  W_N - w);$ \label{rl10}\\
    }{ \label{rl11}
        \For{each child node $N_i$ of N}{ \label{rl12}
        $\mathit{SA} + \gets \mathit{DFSearch\_TVF}(N_i, S, W_{N_i});$ \label{rl13}
        } \label{rl14}
    } \label{rl15}
    \KwRet $\mathit{SA}$; \label{rl16}
\end{algorithm}
\vspace{-0.5cm}
\subsection{Adaptive Algorithm}
In SC, tasks and workers are continuously changing and moving, requiring real-time updates and processing to ensure optimal assignment. This section describes the procedure of an adaptive algorithm that adjusts the task assignment based on current and predicted tasks, responding to fluctuations in supply and demand.

Algorithm~\ref{lta} outlines the complete process of the adaptive algorithm. The input is a continuous stream of arriving workers and tasks, and the output is the corresponding assignment result. When a worker or task appears on the platform, the optimal task planning assignment $\mathit{PA}$ is calculated by Algorithm~\ref{tpa} to achieve the global maximum revenue based on current and future tasks (lines~\ref{lta3}--\ref{lta9}). For each idle worker, the first task in their current planned sequence is executed (lines~\ref{lta10}--\ref{lta14}). Finally, we remove all workers and tasks 
with past deadlines, 
as well as completed tasks (line~\ref{lta15}).
\begin{algorithm}\label{lta}
\small
    \SetAlgoLined 
    \caption{Adaptive Algorithm}
    \SetKwInOut{Input}{Input}
    \SetKwInOut{Output}{Output}
    \Input{A stream of arriving objects $\{\delta_i|\delta_i \in \{w, s\} \}$}
    \Output{Task assignment $A$}
    $\mathit{PA} \gets \emptyset; A \gets \emptyset; W \gets \emptyset; S \gets \emptyset$; \label{lta1}\\
    \For{each arrive object $\delta$}{ \label{lta2}
    
        \eIf{$\delta$ is a worker}{ \label{lta3}
            $W \gets W + \delta$; \label{lta4}\\
            $\mathit{PA} \gets \mathit{TPA}(W, S)$; \label{lta5}
        }{ \label{lta6}
         $S \gets S + \delta$; \label{lta7}\\
        $\mathit{PA} \gets \mathit{TPA}(W, S)$;\label{lta8}
        }\label{lta9}
     
     \For{$\{w_i, \mathit{VR}(w_i)\} \in \mathit{PA}$}{ \label{lta10}
        \If{$w_i$ is ready to accept tasks and $len(\mathit{VR}(w_i)) \ge 1$}{ \label{lta11}
            $A.w_i \gets A.w_i + \mathit{VR}(w_i)[0]$; \label{lta12}
        }\label{lta13}
     }\label{lta14}
     Remove all workers/tasks whose deadlines have passed and completed tasks; \label{lta15}\\
    } \label{lta16}
   
    \KwRet $A$; \label{lta17}
\end{algorithm}

\begin{algorithm}
\small
    \label{tpa}
    \SetAlgoLined 
    \caption{Task Planning Assignment (TPA) Algorithm}
    \SetKwInOut{Input}{Input}
    \SetKwInOut{Output}{Output}
    \Input{A set of workers $W$, a set of tasks $S$}
    \Output{$\mathit{PA}$}

    $opt \gets 0; \mathit{PA} \gets \emptyset;$  \label{tpa1}\\

    \For{each worker $w \in W$}{ \label{tpa2}
    $RS_w \gets$ compute the reachable tasks for $w$; \label{tpa3}\\
    $Q_w \gets$ compute the set of maximal valid task sequences for $w$; \label{tpa4}
    }\label{tpa5}
    $G \gets$ construct worker dependency graph; \label{tpa6}

    \For{each connected graph $g \in G$}{ \label{tpa7}
    $\mathbb{X}_g \gets$ decompose $g$ into vertex cluster; \label{tpa8}\\
    $N_g \gets$ organize $\mathbb{X}_g$ into a tree; \label{tpa9}\\
    $ \mathit{PA} + \gets \mathit{DFSearch\_TVF}(N_g, S, W_{N_g});$ \label{tpa10}
    } \label{tpa11}
    \KwRet $\mathit{PA}$; \label{tpa12}
    
\end{algorithm}

Next, we elaborate on the details of the task planning assignment in Algorithm~\ref{tpa}. The inputs are the worker set $W$ and the task set $S$. The output is the optimal task planning assignment $\mathit{PA}$, which aims to maximize the total expected revenue from both current and future tasks.
We compute the reachable task set $\mathit{RS}_w$ and the set of maximal valid task sequence $Q_w$ for each worker $w$ (lines~\ref{tpa3}--\ref{tpa4}), followed by constructing the worker dependency graph (line~\ref{tpa6}). Then, for each connected component $g \in G $, we decompose $g$ into a set of vertex clusters using the MCS algorithm (line~\ref{tpa8}) and organize them into a tree with the RTC algorithm (line~\ref{tpa9}). Finally, we apply the $\mathit{DFSearch\_TVF}$ algorithm to find the optimal assignment for each sub-problem (line~\ref{tpa10}). Finally, we can get the optimal task assignment (line~\ref{tpa12}).

%% file: worker_dependency_separation.tex
\subsection{Worker Dependency Separation}
The primary computational challenge lies in enumerating all possible valid task combinations for each worker, leading to an exponentially expanding search space as the number of workers and tasks increases. In practice, however, a worker typically shares tasks with only a limited number of other workers who have similar or intersecting travel routes. To address this issue, we first construct a worker dependency graph and apply a graph partitioning method. We then organize the workers within each subgraph into a tree structure, ensuring that workers in sibling nodes are independent of one another. 

\subsubsection{Maximal Valid Task Sequence Generation} In this section, we first find reachable tasks, based on which the maximal valid task sequences for each worker are generated.

\textbf{Finding Reachable Tasks.}
Due to the constraints of workers' reachable distance, offline time, and tasks' expiration time, each worker $w$ can only complete a subset of tasks in their availability window $T_w$ from the current assignment time to their offline time. 
The reachable task subset for a worker $w$, denoted by $RS_w$, should satisfy the following constraints: $\forall s \in RS_w$,
\begin{enumerate}[i.]
    \item the task can be completed before its expiration time, i.e., $c(w.l, s.l) \le s.e-t_{now}$ (where $c(a, b)$ denotes the travel time between location $a$ and location $b$, and $t_{now}$ is the current time) and 
    \item the task can be completed within $T_w$ time interval, i.e., $c(w.l, s.l) \le T_w$ and 
    \item all the tasks in this sequence are located in the reachable range of $w$, i.e., $td(w.l, s.l) \le w.d$, where $td(a, b)$ is the travel distance between location $a$ and $b$.
\end{enumerate}

\textbf{Find Maximal Valid Task Sequence.}
Given the reachable task set $RS_w$ for worker 
$w$, we can derive the valid task sequence set $\mathbb{VR}(\mathit{RS}_w)=\{\mathit{VR}(\mathit{RS}_w)\}$ (see Definition~\ref{VR}). 
When the context is clear, we use $\mathbb{VR}$ to denote $\mathbb{VR}(RS_w)$.
We define the set of Maximal Valid Task Sequences for worker $w$ as $Q_w = \{q_1, q_2, \dots, q_{|Q_w|}\}$, where $q \in \mathbb{VR}$, and it obtains the minimal cost among all sequences in 
$\mathbb{VR}$ that consist of the same set of elements. Specifically, assuming that the order of $q$ is $(s_1, s_2, ..., s_{|q|})$, then $\forall \mathit{VR'} \in \mathbb{VR}$, the arrival time of both sequence should follow Eq.~\ref{maxvr}.
\begin{equation}
t_{q,w}(s_{|q|}) \leq  t_{\mathit{VR'},w}(s_{|\mathit{VR'}|}),
\label{maxvr}
\end{equation}
where the sequences $q$ and $\mathit{VR'}$ contain the same set of elements in different orders.

\subsubsection{Worker Dependency Graph Construction}
Given a worker set $W$ and a task set $S$, we can construct a Worker Dependency Graph (WDG), $G(W, E)$, where each node represents a worker. An edge $e(u,v) \in E$ exists between nodes $u$ and $v$ if the corresponding workers are dependent on each other. Two workers are considered dependent if they share the same reachable tasks; otherwise, they are independent. The time complexity of WDG construction is $O(|W|^2 \cdot |RS|)$, where $|RS|$ is the average number of reachable tasks for each worker.

\subsubsection{Graph Partition}
In this part, we use the maximum cardinality search (MCS) algorithm~\cite{MCS} to iteratively find maximal cliques. MCS consists of the following two steps: 

\begin{enumerate}[i.]
    \item Given a worker dependency graph (WDG), add appropriate new edges to create a corresponding chordal graph. A chordal graph is a graph in which every cycle of four or more vertices has at least one chord, where a chord is an edge that connects two non-adjacent vertices in the cycle.
    \item Find all maximal cliques in the chordal graph.
\end{enumerate}

\subsubsection{Tree Construction}
In this step, our objective is to organize the groups of workers in a tree structure so that the workers in sibling nodes are independent of each other. This setup allows us to independently solve the optimal assignment sub-problem for each sibling node. To achieve this, we use the following Recursive Tree Construction (RTC)~\cite{zhao2019destination} algorithm.

\begin{enumerate}[i.]
    \item Try to remove the cliques $X_i \in X$ (generated by the graph partition step) from the WDG, and the graph will be separated into several components. Select the cliques $X'$ that result in the highest number of components. Consider $X'$ as the parent node for each output of the recursive procedure in the next step.
    \item Apply the MCS algorithm to each sub-graph by removing workers of $X'$ and recursively applying the algorithm to the output of the MCS algorithm.
    \item Return $X'$ as the root node of this sub-tree.
\end{enumerate}

Given the worker dependency graph WDG, we construct a tree structure from the RTC algorithm, denoted by $\Gamma$(with a set of nodes $\mathbb{N}_\Gamma = \{N_1, N_2, ..., N_{|\mathbb{N}_\Gamma|}\}$), that satisfies the following properties:
\begin{enumerate}[i.]
    \item $\cup_{i \in |\mathbb{N}_\Gamma|}N_i = W$; and
    \item workers in sibling nodes are independent of each other.
\end{enumerate}

The time complexity of RTC in the $i^{th}$ recursion is $O(|X^i|+|G^i| \cdot (|V^i|+|E^i|))$, where $|X^i|$ represents the size of cliques generated by the graph partition step, $G^i$ is the subgraph set by performing step 1 in the $i^{th}$ recursion, $|E^i|$ is the number of edges in the chordal graph obtained and $|V^i|$ is the number of nodes in that chordal graph.

%% file: experiment.tex
\section{Experimental Evaluation}
\label{V}
In this section, we conduct experiments to evaluate the effectiveness and efficiency of our proposed methods on two real datasets. All the experiments are implemented on an AMD Ryzen 7 CPU 3.20 GHz with 16GB RAM.
\subsection{Experimental Setup}
The experiments are conducted using two ride-hailing datasets: Yueche and DiDi\footnote{https://github.com/Yi107/Dataset-for-PCOM}.
For the Yueche dataset, generated between 9:00 and 11:00 on November 1st, 2016, each worker and task is associated with key information such as location, start time, due time, and the worker's reachable distance. The DiDi dataset has a similar information description akin to the Yueche dataset but covers the period from 21:00 to 23:00 on November 1st, 2016. 
In both datasets, the driver and passenger matches are used to simulate our problem, where we assume that passengers are tasks and drivers are workers in the SC system since drivers who pick up passengers at different locations may be good candidates to perform spatial tasks in the vicinity of those locations. The locations of workers correspond to where they are informed to pick up passengers. For each passenger, we use the location of the passenger and the time of the pick-up request as the task's location and publication time, respectively.
We use the data from the preceding hour (i.e., from 8:00 to 9:00 in DiDi dataset and from 20:00 to 23:00 in Yueche dataset) as historical data to train the task demand prediction model.
Table~\ref{tab:dataset} provides detailed information about these two real datasets, and Table~\ref{tab:experiment_parameters} shows our experimental settings, where the default values of all parameters are underlined.

\begin{table}[htbp]
    \centering
    \vspace{-0.5cm}
    \caption{Real datasets}
    \vspace{-0.2cm}
    \begin{tabular}{|c|c|c|c|c|}
    \hline
        Dataset& $|W|$ & $|S|$  & Time range & Region  \\ \hline
         Yueche &  624& 11,052  & 9:00 - 11:00& Chengdu \\ \hline
         DiDi & 760& 8,869  & 21:00 - 23:00 & Chengdu \\ \hline
    \end{tabular}
    \vspace{-0.5cm}
    \label{tab:dataset}
\end{table}

\begin{table}[htbp]
    \centering
    \vspace{-0.2cm}
    \caption{experiment parameters}
    \vspace{-0.2cm}
    \begin{tabular}{ll}
        \hline Parameters & Values \\ \hline
        Time interval $\Delta T$ (s) &  \underline{5}, 6, 7, 8, 9\\
        Number of tasks $|S|$ (Yueche) &  7K, 8K, 9K, 10K, \underline{11K}\\
        Number of tasks $|S|$ (DiDi) &   5K, 6K, 7K, 8K \underline{9K}\\
        Number of workers $|W|$ (Yueche) & 200, 300, 400, 500, \underline{600}\\
        Number of workers $|W|$ (DiDi) &  300, 400, 500, 600, \underline{700}\\
        Reachable distance of workers (km) & 0.05, 0.1, 0.5, \underline{1}, 5 \\
        Available time of workers $\mathit{off}-\mathit{on}$ (h) & 0.25, 0.5, 0.75, \underline{1}, 1.25 \\
        Valid time of tasks $e-p$ (s) & 10, 20, 30,  \underline{40}, 50 \\
        \hline
    \end{tabular}
    \vspace{-0.5cm}
\label{tab:experiment_parameters}
\end{table}

\subsection{Experimental Results}
\subsubsection{Performance of Task Demand Prediction}
We evaluate the performance of the task demand prediction phase and its impact to the subsequent task assignment. We choose $80\%$ location data of workers/tasks for training and $20\%$ for testing.

\textbf{Evaluation Methods.} We study the performance of the following methods.

\begin{enumerate}[i.]
\item LSTM~\cite{lstm}: A Long Short-Term Memory model featuring a fully connected layer and an activation function.
\item Graph-Wavenet~\cite{wu2019graph}: A spatial-temporal graph convolutional network, which integrates diffusion graph convolutions with 1D dilated convolutions.
\item DDGNN: Our Dynamic Dependency Graph Neural Network, which is based on multivariate time series learning.
\end{enumerate}

\textbf{Metrics.} To evaluate the accuracy of task demand prediction, we adopt the metric, Average Precision (AP), an important indicator used to measure the overall performance of the predictor at different thresholds. AP is calculated based on Precision and Recall. Precision is the ratio of correctly predicted positive examples (true positives) to all predicted positive examples, denoted by $\mathit{Precision} = \frac{\mathit{TP}}{\mathit{TP}+\mathit{FP}}$, where $\mathit{TP}$ denotes the number of true positives and   $\mathit{FP}$ denotes the number of false positive examples. Recall is the ratio of correctly predicted positive examples to all actual positive examples, 
denoted by $\mathit{Recall}=\frac{\mathit{TP}}{\mathit{TP}+\mathit{FN}}$, where $\mathit{FN}$ denotes the number of false negative examples. To calculate AP, we calculate Precision and Recall at each threshold and then integrate the area under the Precision-Recall curve to get the AP value.
The accuracy threshold is generated from the range $[0, 1]$ with the initial value $0$ and step $0.01$, which means the threshold is $0, 0.01, 0.02, ..., 1$.
We also evaluate the efficiency, including the \emph{training} and \emph{testing
time}.

\begin{figure}[htbp]
    \centering
    \vspace{-0.4cm}
    \setlength{\abovecaptionskip}{-0.1cm}
    \subfigure[Average Precision]{\includegraphics[width=0.48 \linewidth]{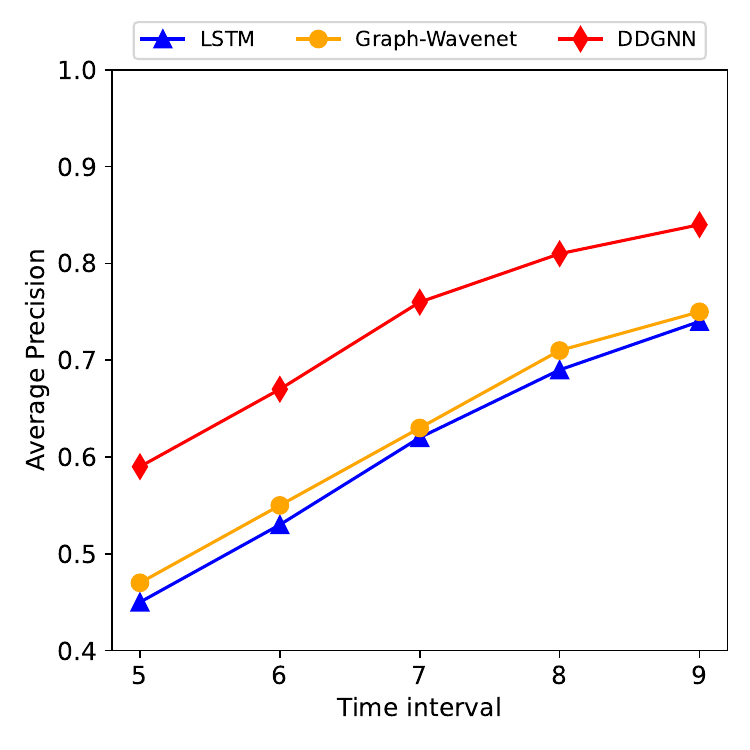}
		\label{fig:pred_AP_yc}}    
    \subfigure[Number of Assigned Tasks]{\includegraphics[width=0.48 \linewidth]{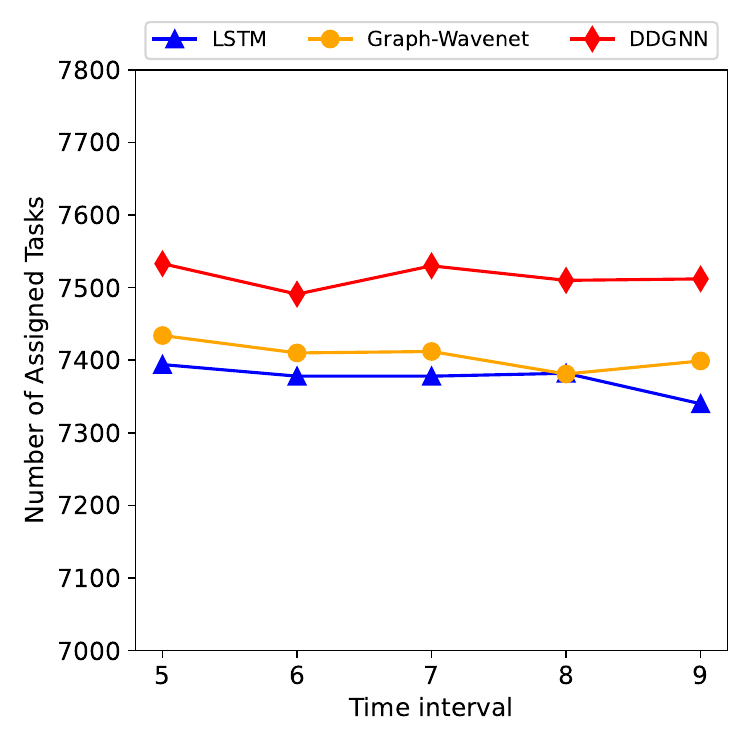}
		\label{fig:pred_R_yc}}
    
    \subfigure[Training Time]{\includegraphics[width=0.48 \linewidth]{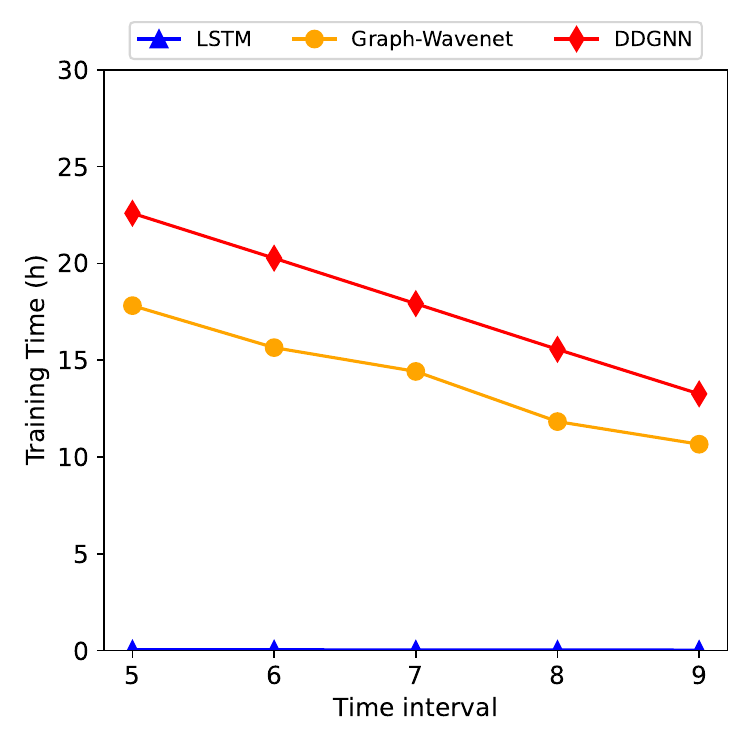}
		\label{fig:pred_train_time_yc}}  
    \subfigure[Testing Time]{\includegraphics[width=0.48 \linewidth]{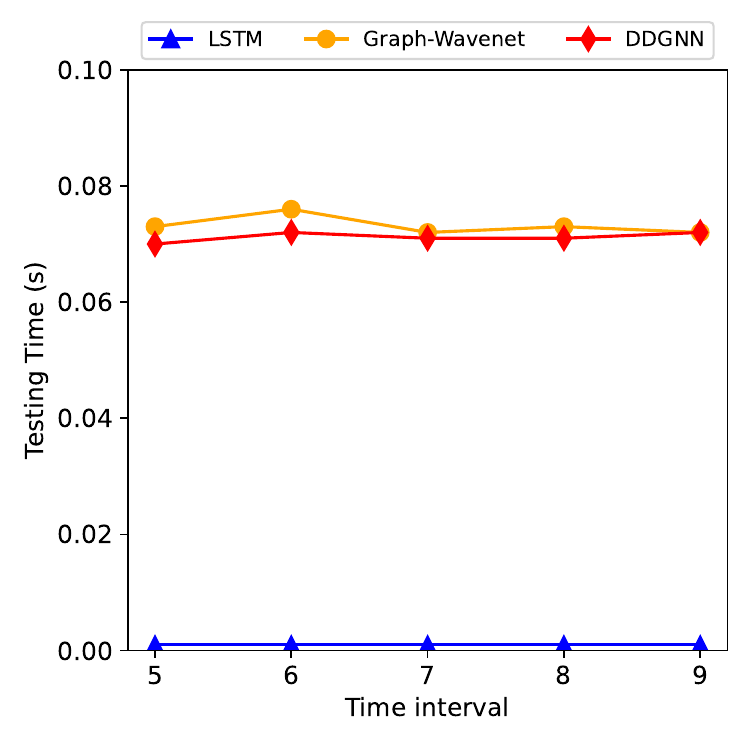}
		\label{fig:pred_test_time_yc}}
  
    \caption{Performance of Task Demand Prediction: Effect of $\Delta T$} on Yueche
    \label{example_time_period_yc}
    
\end{figure}

\begin{figure}[htbp]
    \centering
    \setlength{\abovecaptionskip}{-0.1cm}
    \vspace{-0.5cm}
    \subfigure[Average Precision]{\includegraphics[width=0.48 \linewidth]{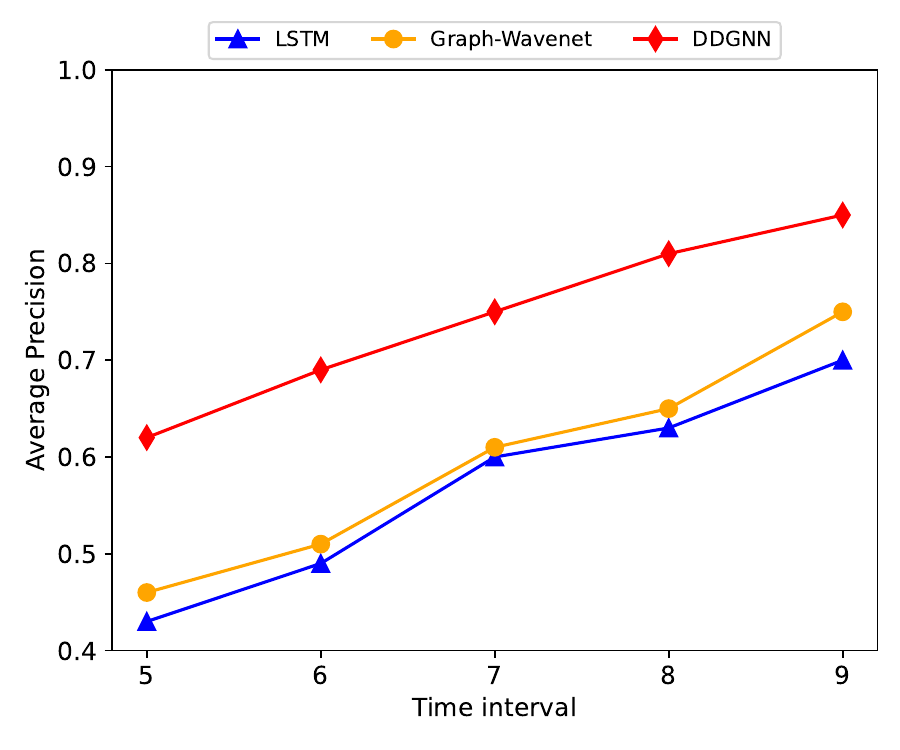}
		\label{fig:pred_AP_dd}}    
    \subfigure[Number of Assigned Tasks]{\includegraphics[width=0.48 \linewidth]{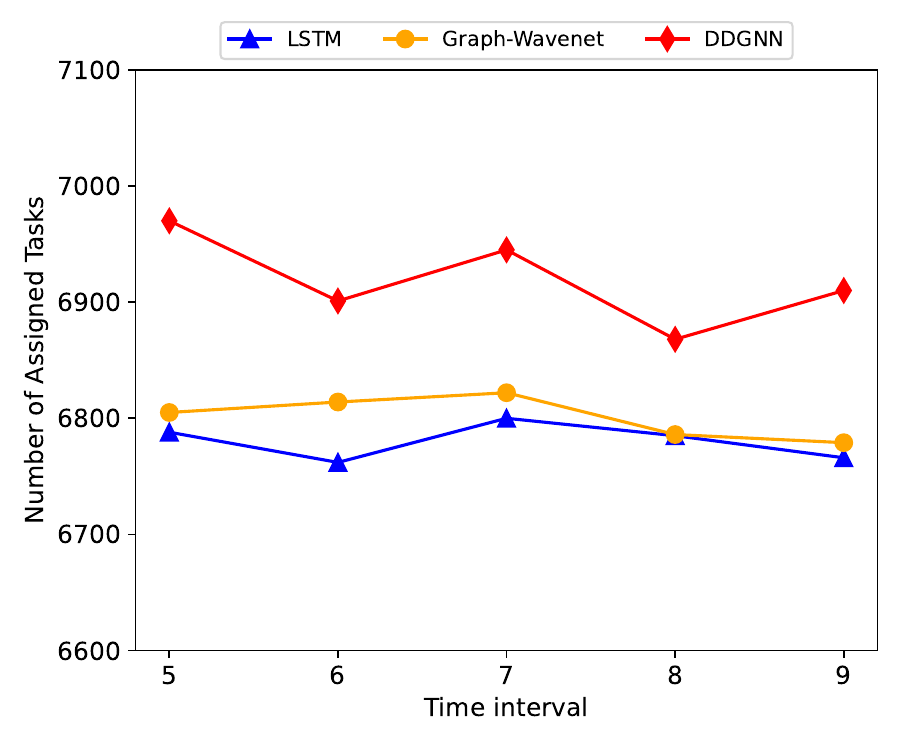}
		\label{fig:pred_R_dd}}
    
    \subfigure[Train Time]{\includegraphics[width=0.48 \linewidth]{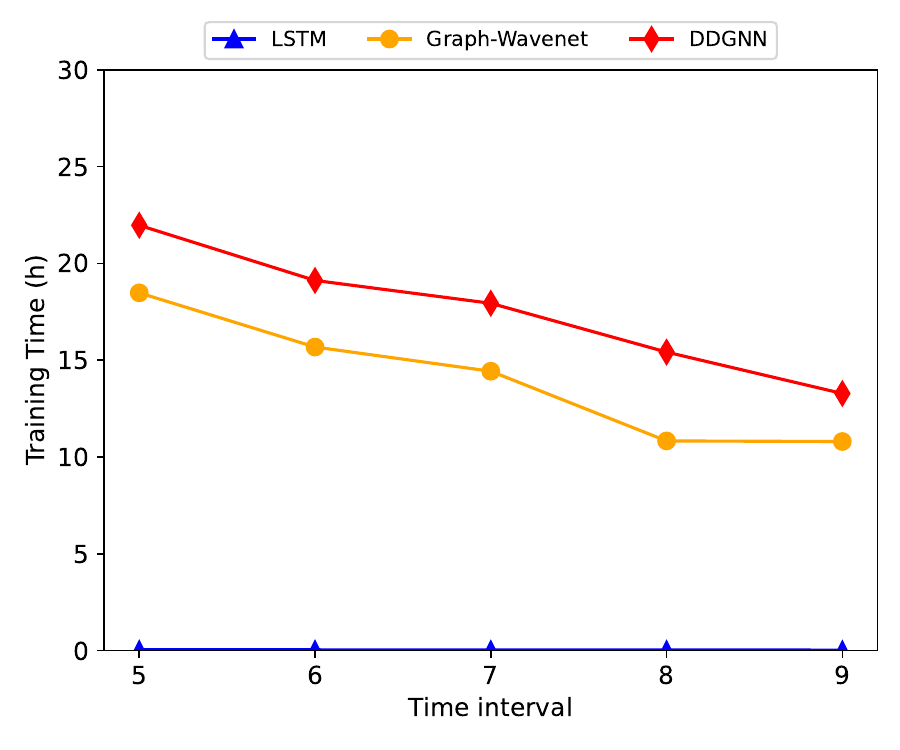}
		\label{fig:pred_train_time_dd}}  
    \subfigure[Test Time]{\includegraphics[width=0.48 \linewidth]{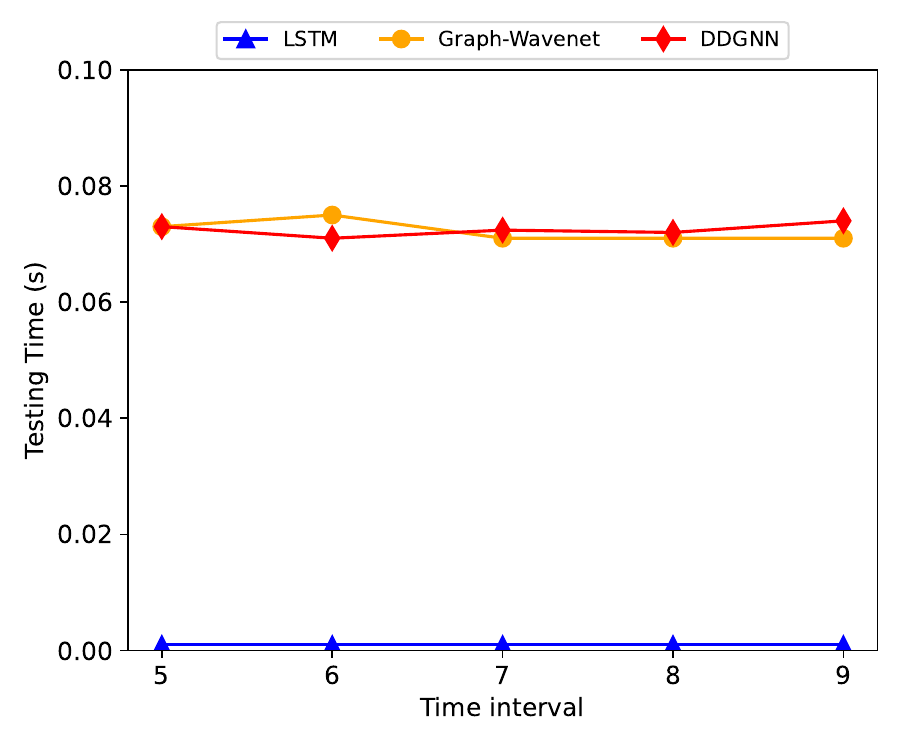}
		\label{fig:pred_test_time_dd}}
  
    \caption{Performance of Task Demand Prediction: Effect of $\Delta T$} on DiDi
    \label{example_time_period_dd}
    \vspace{-0.6cm} 
\end{figure}

\vspace{-0.3cm}
\textbf{Effect of $\Delta T$.} In the first set of experiments, we vary the time interval $\Delta T$  and study its effect on task demand prediction. As shown in Figs.~\ref{fig:pred_AP_yc} and~\ref{fig:pred_AP_dd}, the average precision for all approaches shows a similar increasing trend as $\Delta T$ grows.
Regardless of the time intervals, DDGNN  consistently achieves the highest average precision, followed by Graph-Wavenet and LSTM in both Yueche and DiDi datasets, which demonstrates the superiority of DDGNN for predicting task demand. 
The task assignment results of all methods are not affected by the time intervals, as shown in Figs.~\ref{fig:pred_R_yc} and~\ref{fig:pred_R_dd}. However, the assignment results heavily depend on the average precision for a specific $\Delta T$, as higher average precision generally correlates with more accurately predicted tasks. DDGNN outperforms all other methods across all values of $\Delta T$, confirming the effectiveness of our proposed algorithm.
In Figs.~\ref{fig:pred_train_time_yc} and~\ref{fig:pred_train_time_dd}, the training time decreases with longer time interval $\Delta T$, due to the corresponding reduction in training data.
Figs.~\ref{fig:pred_test_time_yc} and~\ref{fig:pred_test_time_dd} show that the testing time remains 
constant across different time intervals, because the model parameters are fixed.

\subsubsection{Performance of Task Assignment}
In this section, we evaluate the effectiveness and efficiency of the task assignment algorithms.

\textbf{Evaluation Methods.} We study the following methods.

\begin{enumerate}[i.]
    \item Greedy: The Greedy task assignment method that assigns each worker the maximal valid task set from the unassigned tasks until all the tasks are assigned or all the workers are exhausted.
    \item FTA: The Fixed Task Assignment method that involves assigning each worker 
    a fixed, predetermined sequence of tasks to be completed in order while satisfying spatio-temporal constraints, which utilizes the worker dependency separation and DFSearch techniques for task assignments.
    \item DTA: The Dynamic Task Assignment method that dynamically adjusts the task sequence for each worker in real-time according to the spatio-temporal distributions of workers and tasks, without relying on prediction, which also employs the worker dependency separation and DFSearch techniques to refine task assignments.
    \item DTA+TP: The DTA method that assigns tasks based  on the task demand prediction.
    \item DATA-WA: The DTA+TP method integrates the task value function (TVF) into task assignment.
    
\end{enumerate}

\textbf{Metrics.} 
Two main metrics are compared for the above methods, i.e.,
the total number of assigned tasks and CPU time for finding task assignments. Specially, 
the CPU time is the average time cost of performing task assignment at each time instance.

\begin{figure}[t]
    \centering
    \setlength{\abovecaptionskip}{-0.1cm}
    \vspace{-0.3cm}
    \subfigure[Number of Assigned Tasks (Yueche)]{\includegraphics[width=0.48 \linewidth]{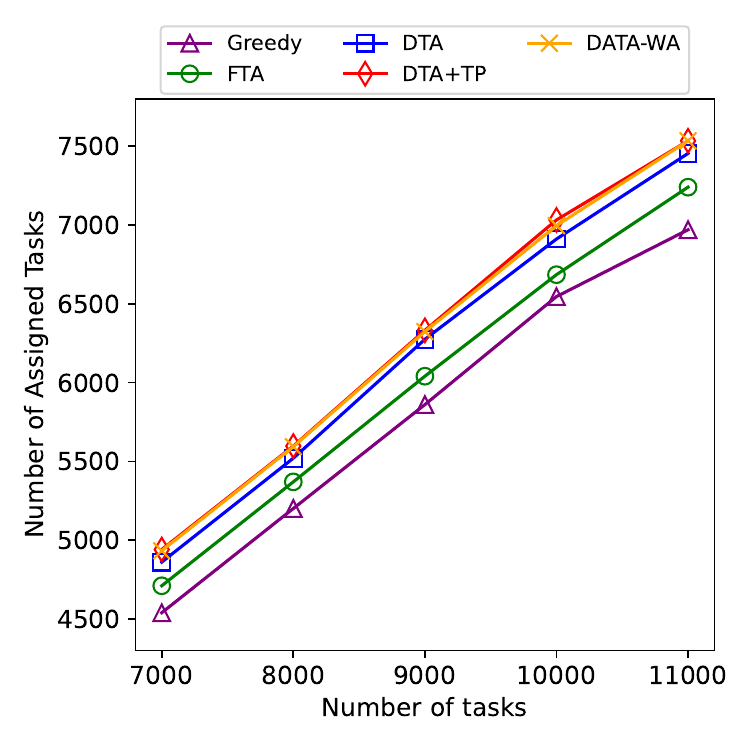}
		\label{fig:t_yc_R}}    
    \subfigure[CPU Cost (Yueche)]{\includegraphics[width=0.48 \linewidth]{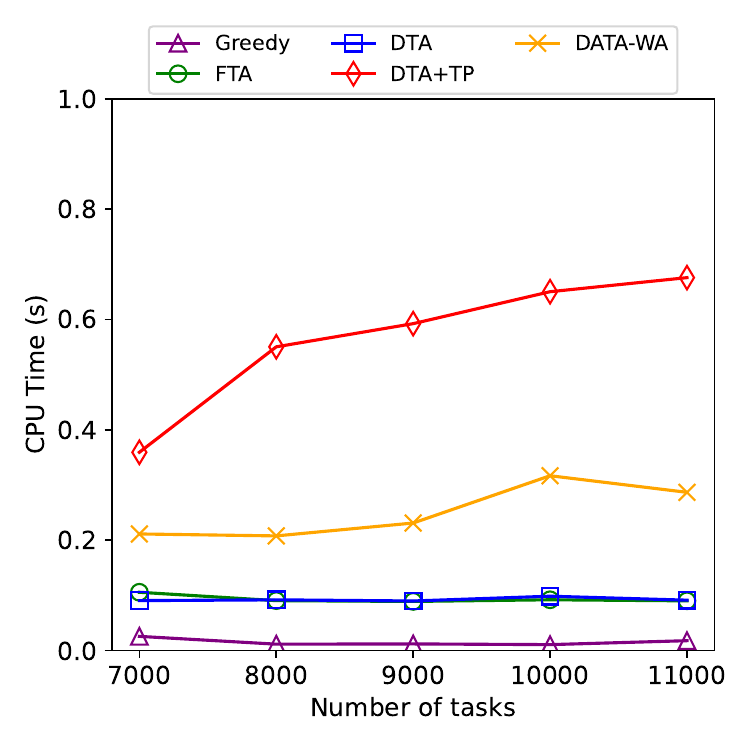}
		\label{fig:t_yc_cpu}}
    
    \subfigure[Number of Assigned Tasks (DiDi)]{\includegraphics[width=0.48 \linewidth]{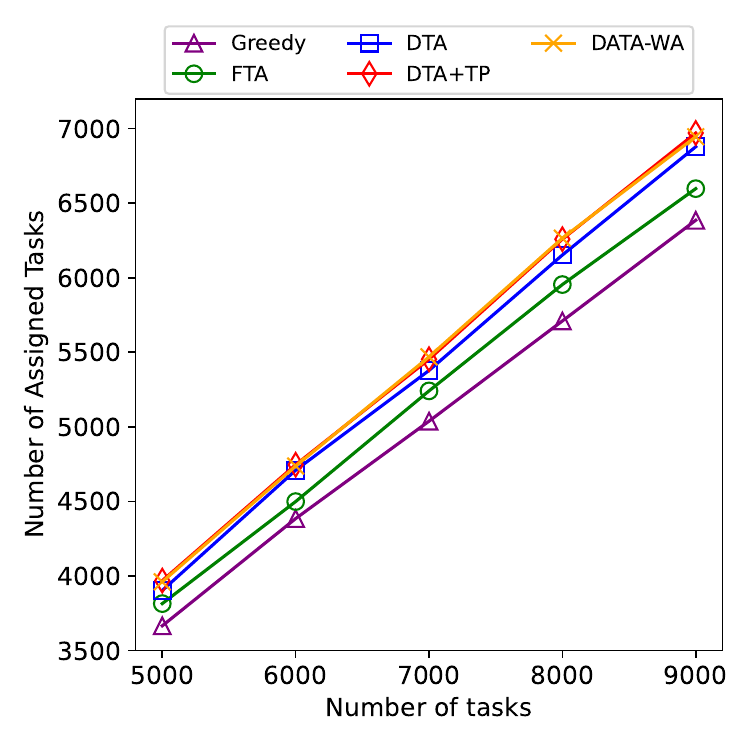}
		\label{fig:t_dd_R}}  
    \subfigure[CPU Cost (DiDi)]{\includegraphics[width=0.48 \linewidth]{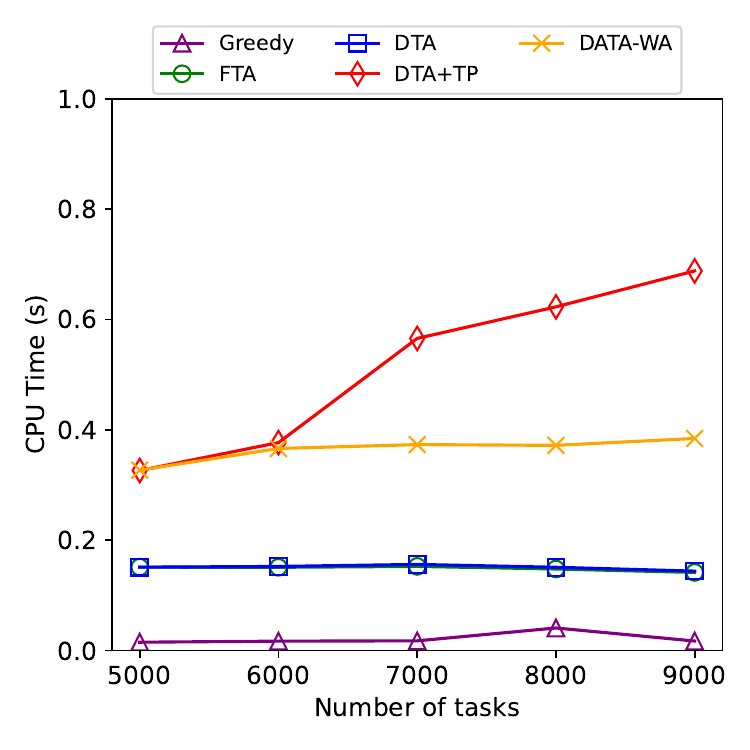}
		\label{fig:t_dd_cpu}}
  
    \caption{Performance of Task Assignment: Effect of $|S|$}
    \label{example_t}
    \vspace{-0.6cm}
\end{figure}

\textbf{Effect of $|S|$.} To study the scalability of the compared methods, we generate five datasets containing 7,000 to 11,000 tasks by randomly selecting from the Yueche dataset and five datasets containing 5,000 to 9,000 tasks from the DiDi dataset. In Figs.~\ref{fig:t_yc_R} and \ref{fig:t_dd_R}, as $|S|$ increases, 
it is  more likely that each worker will be assigned more tasks, 
resulting in an increase in the number of assigned tasks across all methods for both the Yueche and DiDi datasets.  
DTA+TP outperforms the others, demonstrating the effectiveness of our proposed method. In terms of CPU cost, as shown in Figs.~\ref{fig:t_yc_cpu} and \ref{fig:t_dd_cpu}, the CPU time for DTA, Greedy and FTA remains 
relatively constant regardless of $|S|$.
However, for the DTA+TP and DATA-WA methods, the CPU time exhibits an increasing trend with $|S|$ due to the need to search more tasks. 
DATA-WA achieves a similar number of assigned tasks as DTA+TP with less computation, highlighting the effectiveness of RL-based optimization.

\textbf{Effect of $|W|$.} To study the effect of $|W|$, we generate five datasets containing 200 to 600 workers by random selection from the Yueche dataset and five datasets containing 300 to 700 workers from the DiDi dataset. In Figs.~\ref{fig:w_yc_R} and~\ref{fig:w_dd_R}, as $|W|$ increases, more tasks can be assigned to more workers. 
Thus, the number of assigned tasks of all methods increases in both Yueche and DiDi datasets. DTA+TP results in 
the highest number of assigned tasks, followed by DATA-WA, DTA, FTA and Greedy. 
DTA assigns more tasks than FTA since workers can adjust their task sequence when demand and supply change. 
Although DTA+TP assigns more tasks than DATA-WA, it is more time-consuming than DATA-WA, as shown in 
Figs.~\ref{fig:w_yc_cpu} and~\ref{fig:w_dd_cpu}.
The CPU time of DATA-WA is only $42.4\%$--$65.7\%$ of that of DTA+TP.
Greedy is the fastest algorithm and is almost unaffected by $|W|$, but it assigns the fewest tasks. 
Although Greedy and FTA are more efficient than our proposed approaches, they assign fewer tasks. DATA-WA achieves the best balance between task assignment and computational cost compared to other methods.

\begin{figure}[t]
    \setlength{\abovecaptionskip}{-0.1cm}
    \centering
    \vspace{-0.5cm}
    \subfigure[Number~of Assigned Tasks (Yueche)]{\includegraphics[width=0.48 \linewidth]{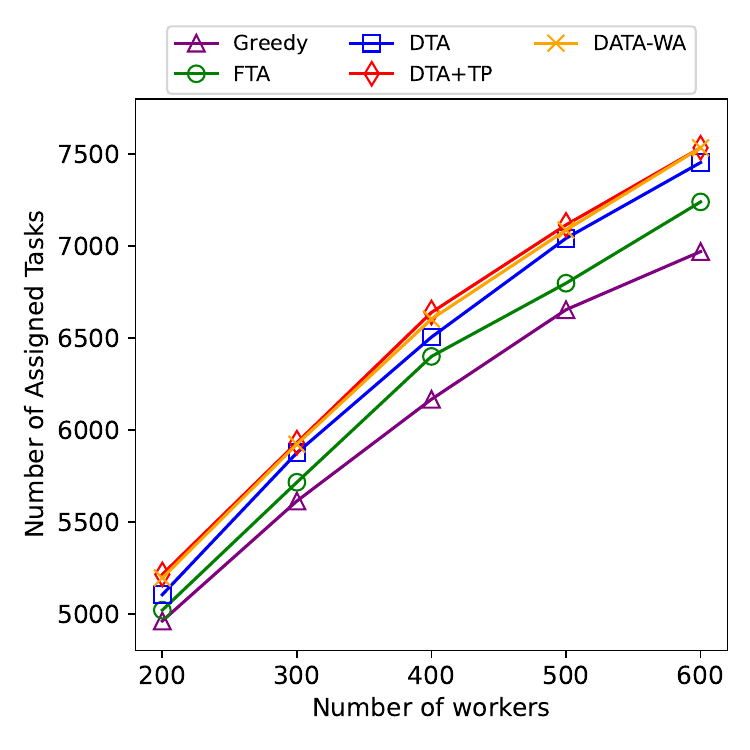}
		\label{fig:w_yc_R}}    
    \subfigure[CPU Cost (Yueche)]{\includegraphics[width=0.48 \linewidth]{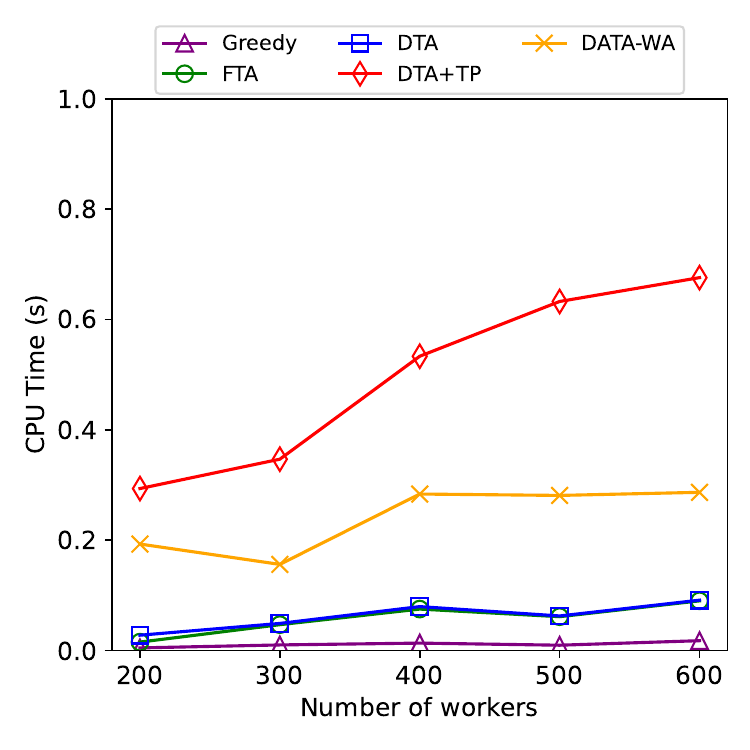}
		\label{fig:w_yc_cpu}}

    \subfigure[Number of Assigned Tasks (DiDi)]{\includegraphics[width=0.48 \linewidth]{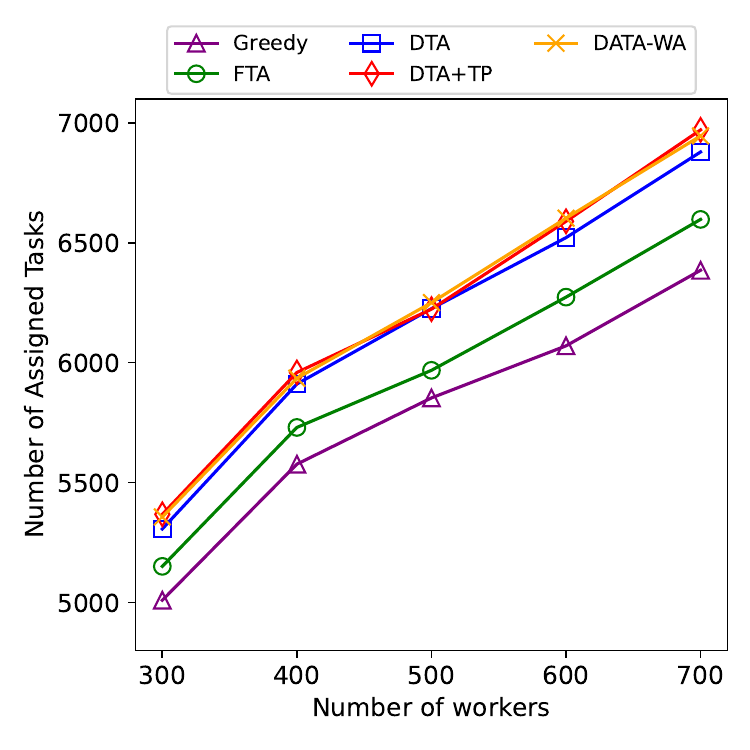}
		\label{fig:w_dd_R}}  
    \subfigure[CPU Cost (DiDi)]{\includegraphics[width=0.48 \linewidth]{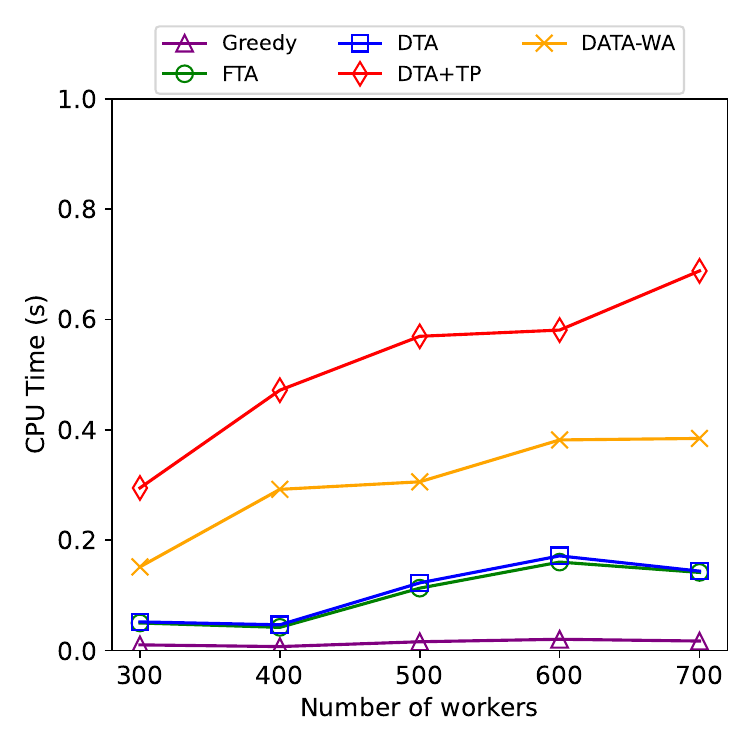}
		\label{fig:w_dd_cpu}}
  
    \caption{Performance of Task Assignment: Effect of $|W|$}
    \label{example_w}
    \vspace{-0.6cm}
\end{figure}

\textbf{Effect of $d$.} We also study the effect of workers’ reachable distance $d$. As shown in Figs.~\ref{fig:dis_yc_R} and \ref{fig:dis_dd_R}, the numbers of assigned tasks generated by all methods have a growing tendency as $d$ increases. This is due to the fact that workers with larger reachable distances 
have more available task assignments.
However, we also notice that the effect of $d$ becomes less significant on all methods when $d$ is greater than or equal to 0.5 km, 
indicating that the methods remain unaffected beyond this threshold.
In addition, when $d$ exceeds 0.5 km, DTA+TP and DATA-WA outperform the others, demonstrating the effectiveness of our proposed algorithms again.
In  Figs.~\ref{fig:dis_yc_cpu} and \ref{fig:dis_dd_cpu}, the CPU cost of all approaches increases with larger reachable distances. This is because that the number of available tasks to be assigned in a given time instance grows when $d$ gets larger, leading to longer computational cost.
\begin{figure}[t]
    \centering
    \setlength{\abovecaptionskip}{-0.1cm}
    \subfigure[Number of Assigned Tasks (Yueche)]{\includegraphics[width=0.48 \linewidth]{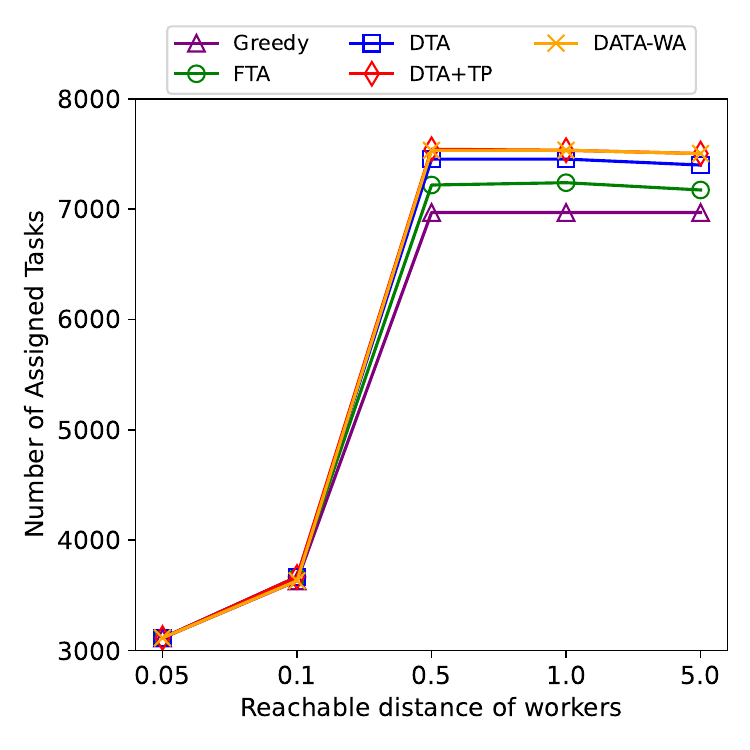}
		\label{fig:dis_yc_R}}    
    \subfigure[CPU Cost (Yueche)]{\includegraphics[width=0.48 \linewidth]{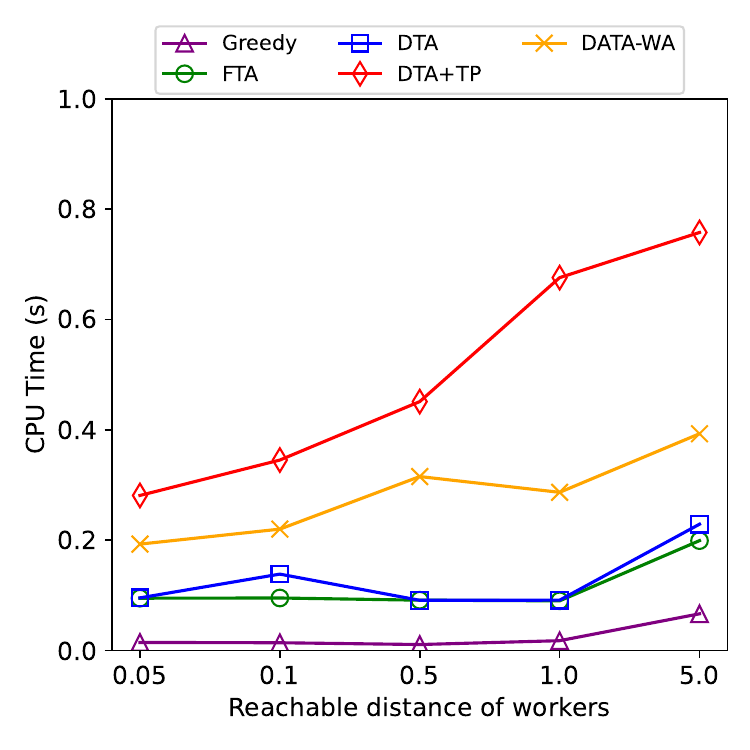}
		\label{fig:dis_yc_cpu}}

    \subfigure[Number of Assigned Tasks (DiDi)]{\includegraphics[width=0.48 \linewidth]{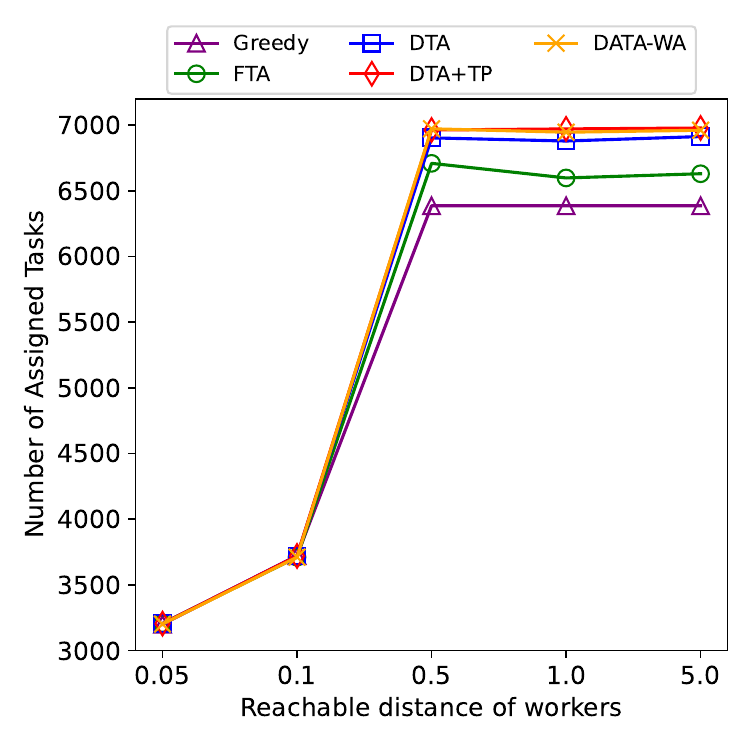}
		\label{fig:dis_dd_R}}  
    \subfigure[CPU Cost (DiDi)]{\includegraphics[width=0.48 \linewidth]{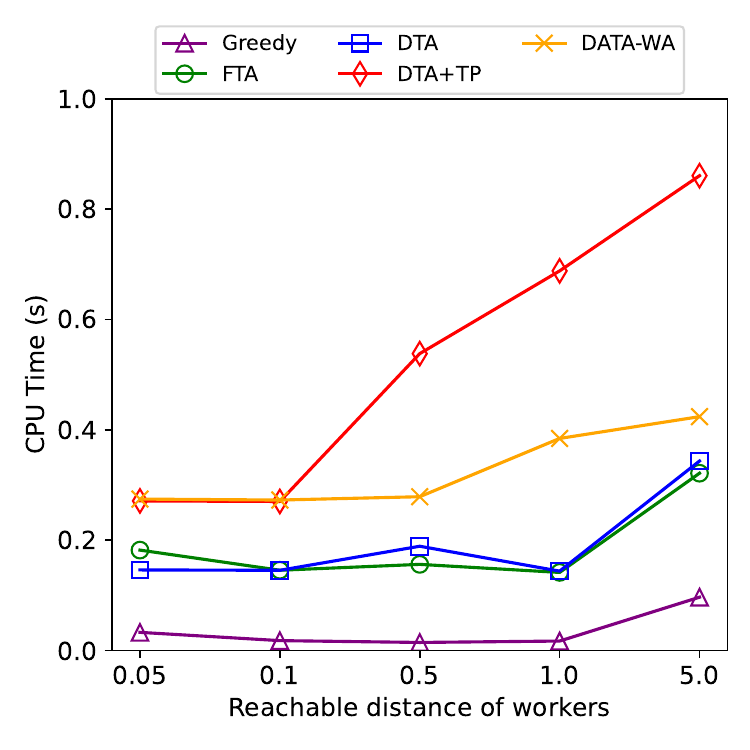}
		\label{fig:dis_dd_cpu}}
     \caption{Performance of Task Assignment: Effect of $d$}
    \label{example_d}
\end{figure}

\begin{figure}[htbp]
    \centering
    \vspace{-0.5cm} 
    \setlength{\abovecaptionskip}{-0.1cm}
    \subfigure[Number of Assigned Tasks (Yueche)]{\includegraphics[width=0.48 \linewidth]{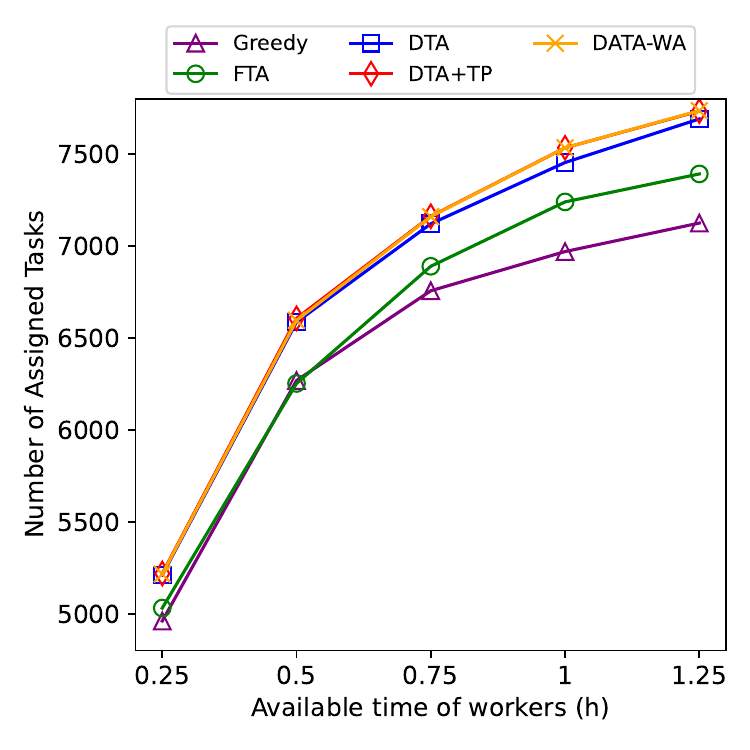}
		\label{fig:wd_yc_R}}    
    \subfigure[CPU Cost (Yueche)]{\includegraphics[width=0.48 \linewidth]{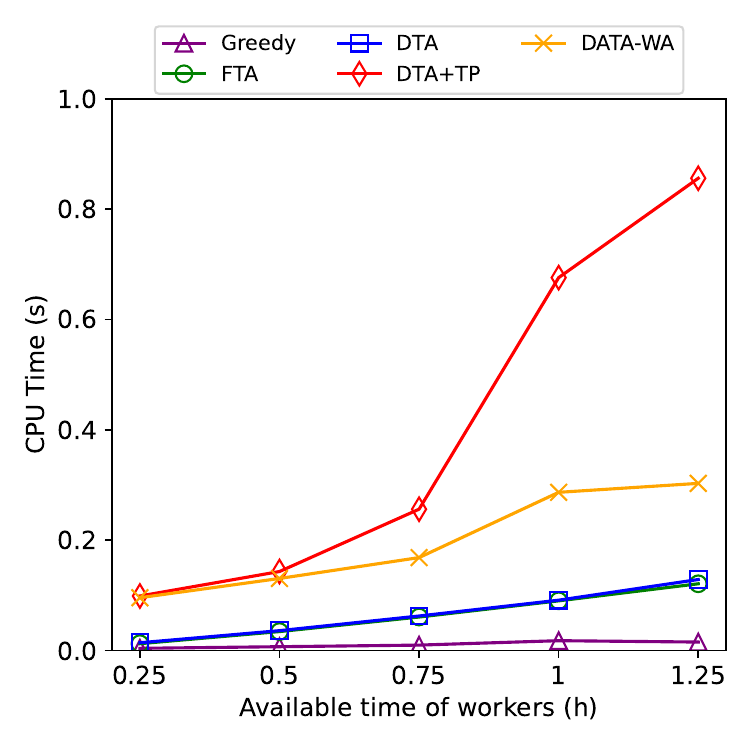}
		\label{fig:wd_yc_cpu}}

    \subfigure[Number of Assigned Tasks (DiDi)]{\includegraphics[width=0.48 \linewidth]{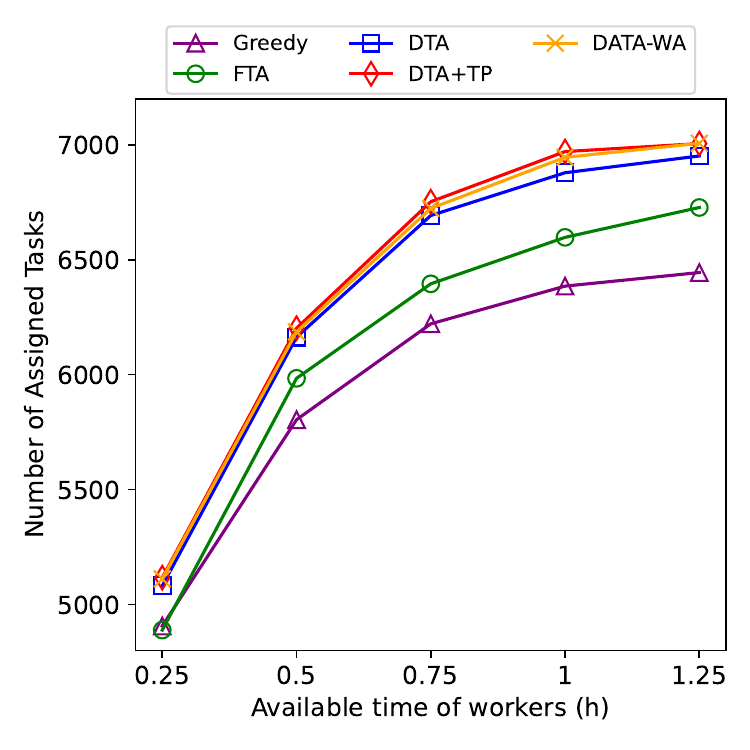}
		\label{fig:wd_dd_R}}  
    \subfigure[CPU Cost (DiDi)]{\includegraphics[width=0.48 \linewidth]{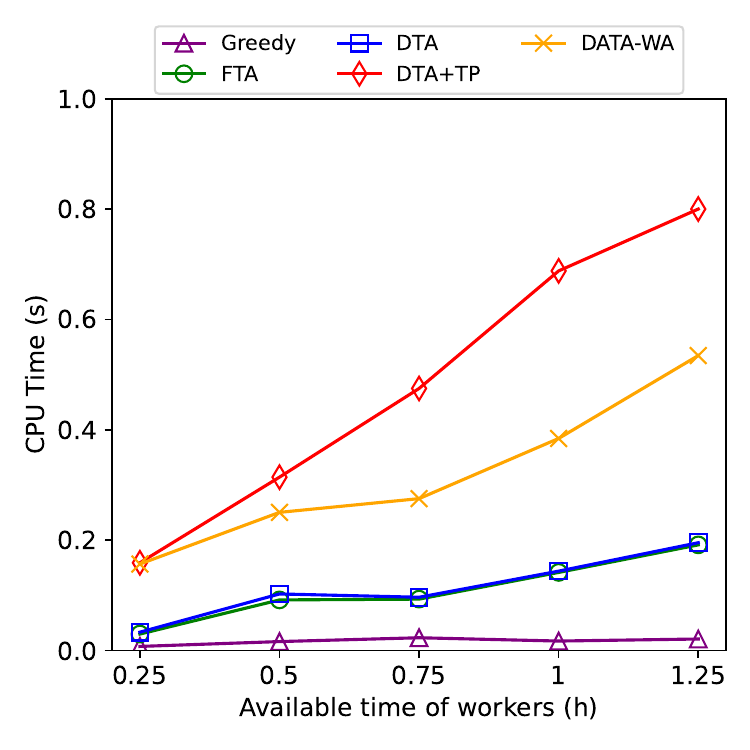}
		\label{fig:wd_dd_cpu}}
  
    \caption{Performance of Task Assignment: Effect of $\mathit{off} - \mathit{on}$}
    \label{example_wd}
    \vspace{-0.8cm} 
\end{figure}

\textbf{Effect of $\mathit{off}-\mathit{on}$.} We further evaluate the effect of workers' available time. As shown in Figs.~\ref{fig:wd_yc_R} and \ref{fig:wd_dd_R}, the number of assigned tasks of all methods gradually increases as the available time of workers increases. This is because that more workers are available for each task. 
DTA+TP and DATA-WA achieve similar numbers of assigned tasks, 
but DATA-WA has significantly lower CPU time compared to DTA+TP, as shown in 
Figs.~\ref{fig:wd_yc_cpu} and \ref{fig:wd_dd_cpu}. 
This reduction in CPU time is due to the task value function minimizing multiple backtracking processes. In addition, the CPU time for all methods gradually increases since the number of available workers for each task increases, resulting in a more extensive search space. 

\textbf{Effect of $e-p$.} We then analyze the impact of the valid time $e-p$ of tasks. As shown in Figs.~\ref{fig:dead_yc_R} and~\ref{fig:dead_dd_R}, the number of assigned tasks increases across all methods as the valid time extends. This is because tasks have a higher probability of being assigned to suitable workers when they have more valid time. Similar to the previous results, our proposed DTA+TP and DATA-WA outperform the others, which confirms the superiority of our proposals. In Figs.~\ref{fig:dead_yc_cpu} and~\ref{fig:dead_dd_cpu}, the CPU times for all approaches increase with longer task valid times, due to the higher number of worker-task assignments.

\begin{figure}[t]
    \centering
    \setlength{\abovecaptionskip}{-0.1cm}
    \vspace{-0.4cm}
    
    \subfigure[Number of Assigned Tasks (Yueche)]{\includegraphics[width=0.48 \linewidth]{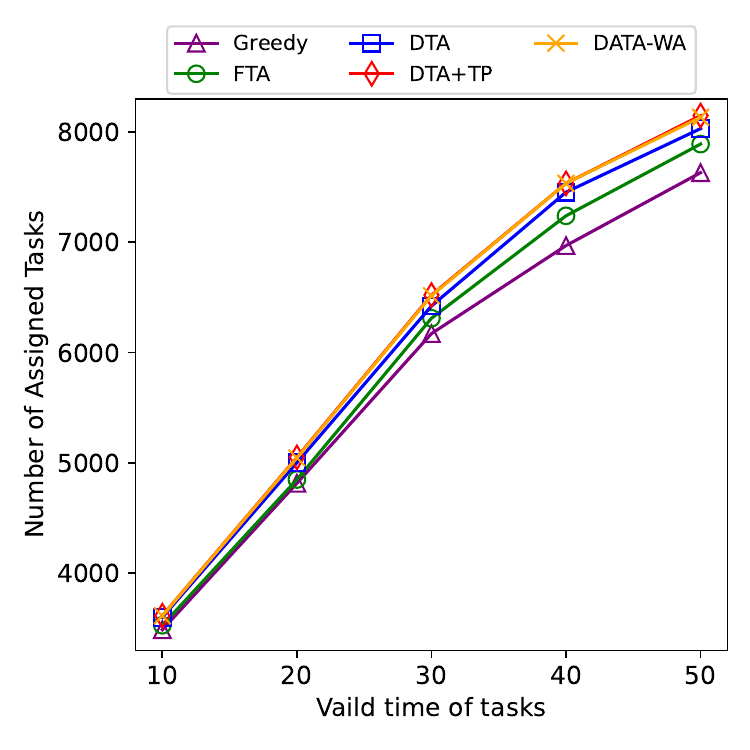}
		\label{fig:dead_yc_R}}    
    \subfigure[CPU Cost (Yueche)]{\includegraphics[width=0.48 \linewidth]{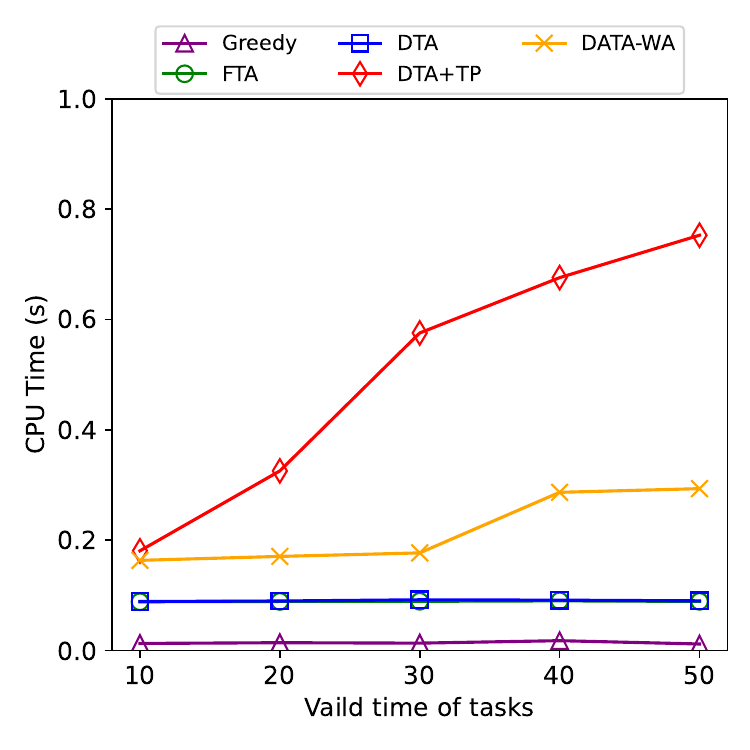}
		\label{fig:dead_yc_cpu}}

    \subfigure[Number of Assigned Tasks (DiDi)]{\includegraphics[width=0.48 \linewidth]{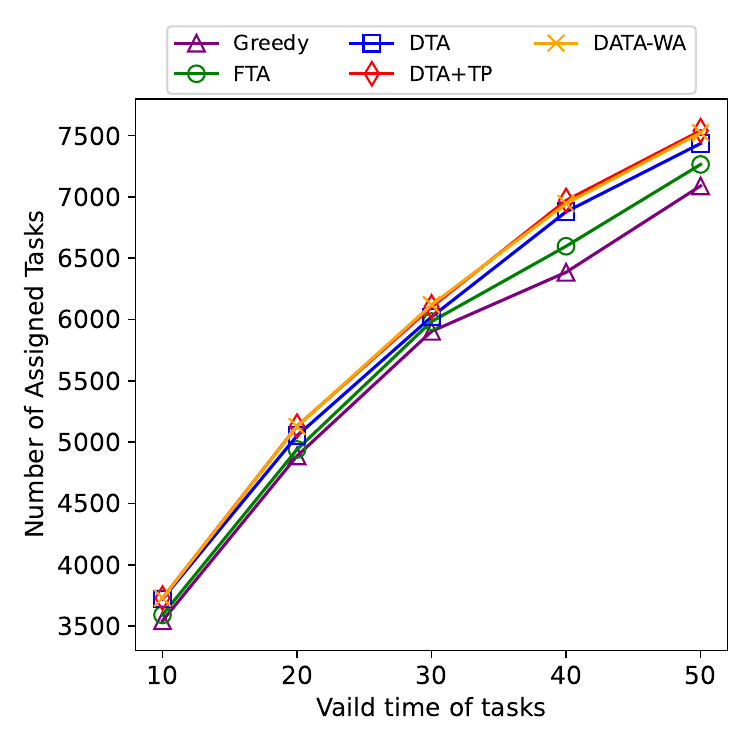}
		\label{fig:dead_dd_R}}  
    \subfigure[CPU Cost (DiDi)]{\includegraphics[width=0.48 \linewidth]{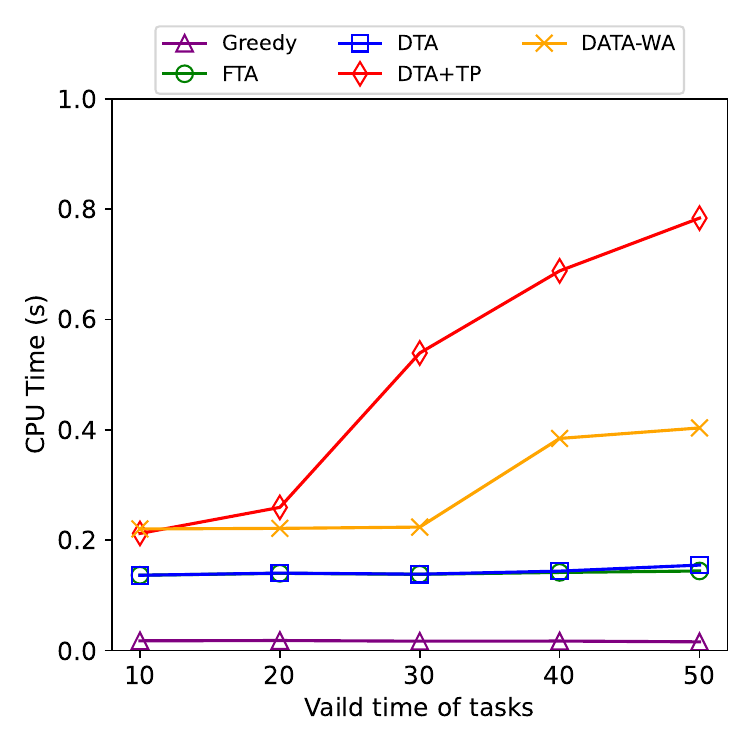}
		\label{fig:dead_dd_cpu}}
  
    \caption{Performance of Task Assignment: Effect of $e-p$}
    \label{example_dd}
    \vspace{-0.8cm}
\end{figure}

%% file: related_work.tex
\section{related work}
\label{VI}
Spatial Crowdsourcing (SC) is an innovative form of crowdsourcing that utilizes smart device carriers as workers who travel to specific locations to complete spatial tasks~\cite{zheng2022crowdsourced, 9474868, 9101624, 9101527, 8731525, zhao2023workerchurn, zhao2021fairness, zhao2019destination, zhao2020predictive, cheng2020real, deng2024task,yanwww2022,liu2022task,ChenZZ22,lai2022loyalty,zhao2023Coalition}.
SC can be categorized based on the task publication mode into Server Assigned Tasks (SAT) mode and Worker Selected Tasks (WST) mode.
In SAT mode, the server assigns tasks to nearby workers with the aim of optimizing system performance. This includes objectives such as maximizing the total number of tasks assigned~\cite{zhao2020predictive, zhao2019destination} or maximizing the overall payoff from these assignments~\cite{zhao2023workerchurn, zhao2021fairness}.
Research in this area has also explored factors like worker preference~\cite{zhao2023preference}, fairness~\cite{zhao2021fairness}, social networks~\cite{cheng2019event, lian2020geography, wang2019mcne}, and worker cooperation~\cite{li2023competition} in task assignment.
Ye et al.~\cite{ye2021task} propose a reinforcement learning based method to achieve the task allocation, and propose a graph neural network method with the attention mechanism to learn the embeddings of allocation centers, delivery points and workers.
Zhao et al.~\cite{zhao2023adataskrec} utilize graph convolutional network to model the geographical distance between out-of-town POIs, generating the out-of-town POI embeddings and to learn workers’ out-of-town preferences.
Wang et al.~\cite{Air-Ground} propose a novel communication-based multi-agent deep reinforcement learning method for data acquisition in urban environments.
Rao et al.~\cite{Can_You2024} propose an online framework by extending multi-agent reinforcement learning with careful augmentation to optimize the profit of order-serving and the data utility of crowdsensing.

However, most of the research conducted so far is based on the assumption of static offline scenarios, where the demand and supply between workers and tasks are known a prior. These studies focus on immediate task assignments, often struggling to adapt to the rapid and unpredictable changes in task demand and worker availability. This highlights the urgent need for further exploration and innovation in this field to ensure optimal task assignments over the long term.

However, SC is a real-time platform where workers and tasks occur dynamically. Recent studies consider task and worker prediction to solve online task assignment problems in SC. Zhai et al.~\cite{zhai2019seqst} propose a novel deep learning model to address the task prediction problem, which captures the temporal dependencies of historical task appearance in sequences at several time scales. Wei et al.~\cite{9816080} propose a location-and-preference joint prediction model to predict workers' locations and preference jointly at each timestamp. They also design a greedy multi-attribute joint task assignment algorithm to maximize the average number of completed tasks under constraints. Peng et al.~\cite{Peng2023} introduce a spatio-temporal prediction strategy that combines a gated recurrent unit and a variational autoencoder for crowdsourcing task prediction. However, the previous studies only concentrate on predicting task distribution, neglecting task demand dynamics, which is essential for accurately forecasting the spatio-temporal distribution of tasks.


To solve the problem of task demand prediction, Yang et al.~\cite{yang2023batch} divide supply and demand into five degrees and use a Markov Predictor to predict the future degree of supply and demand. However, they do not consider the demand dependency in different regions, which is essential for predicting task demand. The closest related research to ours is~the work~\cite{zhao2020predictive}, which hybrids different learning models to predict the locations and routes of future workers and employs a graph embedding approach to estimate the distribution of future tasks but assigns each worker a fixed task sequence. However, in this work, we predict future task demands and update the assigned task sequence in real time for workers to ensure optimal task assignment over long time.

%% file: conclusion.tex
\section{conclusion}
\label{VII}
We propose and offer solutions to a problem termed Adaptive Task Assignment (ATA) in spatial crowdsourcing (SC), which aims to maximize the number of assigned tasks by dynamically adjusting task assignments in response to changing demand and supply. An SC framework, namely \textbf{\underline D}emand-based \textbf{\underline A}daptive \textbf{\underline T}ask \textbf{\underline A}ssignment with dynamic \textbf{\underline W}orker \textbf{\underline A}vailability windows (DATA-WA), is proposed. It consists of two phases:  task demand prediction and task assignment. To the best of our knowledge, our approach is the first to consider the dependency relationships between task demands in different regions. We employ a multivariate time series learning approach to predict future task demands and then adjust task assignments based on real-time and predicted task demand dynamics, as well as dynamic worker availability. An empirical study on real data confirm that our proposed framework significantly improves the effectiveness and efficiency of task assignment. In our future studies, we plan to extend our validation by conducting experiments with human participants or other real-world datasets to understand how our solutions perform in practical scenarios.

%% file: main.bbl
\begin{thebibliography}{10}
\providecommand{\url}[1]{#1}
\csname url@samestyle\endcsname
\providecommand{\newblock}{\relax}
\providecommand{\bibinfo}[2]{#2}
\providecommand{\BIBentrySTDinterwordspacing}{\spaceskip=0pt\relax}
\providecommand{\BIBentryALTinterwordstretchfactor}{4}
\providecommand{\BIBentryALTinterwordspacing}{\spaceskip=\fontdimen2\font plus
\BIBentryALTinterwordstretchfactor\fontdimen3\font minus \fontdimen4\font\relax}
\providecommand{\BIBforeignlanguage}[2]{{%
\expandafter\ifx\csname l@#1\endcsname\relax
\typeout{** WARNING: IEEEtran.bst: No hyphenation pattern has been}%
\typeout{** loaded for the language `#1'. Using the pattern for}%
\typeout{** the default language instead.}%
\else
\language=\csname l@#1\endcsname
\fi
#2}}
\providecommand{\BIBdecl}{\relax}
\BIBdecl

\bibitem{tong2020spatial}
Y.~Tong, Z.~Zhou, Y.~Zeng, L.~Chen, and C.~Shahabi, ``Spatial crowdsourcing: a survey,'' \emph{VLDBJ}, vol.~29, pp. 217--250, 2020.

\bibitem{zheng2022privacy}
L.~Zheng, L.~Chen, and P.~Cheng, ``Privacy-preserving worker allocation in crowdsourcing,'' \emph{VLDBJ}, vol.~31, pp. 733--751, 2022.

\bibitem{yao2023non}
J.~Yao, L.~Yang, Z.~Wang, and X.~Xu, ``Non-rejection aware online task assignment in spatial crowdsourcing,'' \emph{TSC}, vol.~16, pp. 4540--4553, 2023.

\bibitem{li2023acta}
B.~Li, Y.~Cheng, Y.~Yuan, Y.~Yang, Q.~Jin, and G.~Wang, ``Acta: Autonomy and coordination task assignment in spatial crowdsourcing platforms,'' \emph{PVLDB}, vol.~16, pp. 1073--1085, 2023.

\bibitem{zhao2023preference}
Y.~Zhao, J.~Liu, Y.~Li, D.~Zhang, C.~S. Jensen, and K.~Zheng, ``Preference-aware group task assignment in spatial crowdsourcing: Effectiveness and efficiency,'' \emph{TKDE}, vol.~35, pp. 10\,722--10\,734, 2023.

\bibitem{zhao2020predictive}
Y.~Zhao, K.~Zheng, Y.~Cui, H.~Su, F.~Zhu, and X.~Zhou, ``Predictive task assignment in spatial crowdsourcing: a data-driven approach,'' in \emph{ICDE}, 2020, pp. 13--24.

\bibitem{Peng2023}
M.~Peng, J.~Hu, H.~Lin, X.~Wang, P.~Liu, K.~Dev, S.~A. Khowaja, and N.~M.~F. Qureshi, ``Spatiotemporal prediction based intelligent task allocation for secure spatial crowdsourcing in industrial iot,'' \emph{TNSE}, vol.~10, pp. 2853--2863, 2023.

\bibitem{zhao2022Profit}
Y.~Zhao, K.~Zheng, Y.~Li, J.~Xia, B.~Yang, T.~B. Pedersen, R.~Mao, C.~S. Jensen, and X.~Zhou, ``Profit optimization in spatial crowdsourcing: Effectiveness and efficiency,'' \emph{TKDE}, 2022.

\bibitem{zhao2021coalition}
Y.~Zhao, J.~Guo, X.~Chen, J.~Hao, X.~Zhou, and K.~Zheng, ``Coalition-based task assignment in spatial crowdsourcing,'' in \emph{ICDE}, 2021, pp. 241--252.

\bibitem{li2020group}
X.~Li, Y.~Zhao, J.~Guo, and K.~Zheng, ``Group task assignment with social impact-based preference in spatial crowdsourcing,'' in \emph{DASFAA}, 2020, pp. 677--693.

\bibitem{li2020consensus}
X.~Li, Y.~Zhao, X.~Zhou, and K.~Zheng, ``Consensus-based group task assignment with social impact in spatial crowdsourcing,'' \emph{Data Science and Engineering}, vol.~5, pp. 375--390, 2020.

\bibitem{9816080}
X.~Wei, B.~Sun, J.~Cui, and M.~Qiu, ``Location-and-preference joint prediction for task assignment in spatial crowdsourcing,'' \emph{TCAD}, vol.~42, pp. 928--941, 2023.

\bibitem{Wang2021Task}
Z.~Wang, Y.~Zhao, X.~Chen, and K.~Zheng, ``Task assignment with worker churn prediction in spatial crowdsourcing,'' in \emph{CIKM}, 2021, pp. 2070--2079.

\bibitem{li2021preference}
Y.~Li, Y.~Zhao, and K.~Zheng, ``Preference-aware group task assignment in spatial crowdsourcing: A mutual information-based approach,'' in \emph{ICDM}, 2021, pp. 350--359.

\bibitem{kazemi2012geocrowd}
L.~Kazemi and C.~Shahabi, ``Geocrowd: enabling query answering with spatial crowdsourcing,'' in \emph{GIS}, 2012, pp. 189--198.

\bibitem{mc}
D.~S. Hochba, ``Approximation algorithms for np-hard problems,'' \emph{ACM Sigact News}, vol.~28, pp. 40--52, 1997.

\bibitem{lstm}
A.~Graves and A.~Graves, ``Long short-term memory,'' \emph{Supervised sequence labelling with recurrent neural networks}, pp. 37--45, 2012.

\bibitem{cui2021metro}
Y.~Cui, K.~Zheng, D.~Cui, J.~Xie, L.~Deng, F.~Huang, and X.~Zhou, ``Metro: a generic graph neural network framework for multivariate time series forecasting,'' \emph{PVLDB}, vol.~15, pp. 224--236, 2021.

\bibitem{deng2025million}
L.~Deng, T.~Wang, Y.~Zhao, and K.~Zheng, ``Million: A general multi-objective framework with controllable risk for portfolio management,'' \emph{PVLDB}, 2025.

\bibitem{chen2021daemon}
X.~Chen, L.~Deng, F.~Huang, C.~Zhang, Z.~Zhang, Y.~Zhao, and K.~Zheng, ``Daemon: Unsupervised anomaly detection and interpretation for multivariate time series,'' in \emph{ICDE}, 2021, pp. 2225--2230.

\bibitem{wu2020connecting}
Z.~Wu, S.~Pan, G.~Long, J.~Jiang, X.~Chang, and C.~Zhang, ``Connecting the dots: Multivariate time series forecasting with graph neural networks,'' in \emph{SIGKDD}, 2020, pp. 753--763.

\bibitem{wu2019graph}
Z.~Wu, S.~Pan, G.~Long, J.~Jiang, and C.~Zhang, ``Graph wavenet for deep spatial-temporal graph modeling,'' in \emph{IJCAI}, 2019, p. 1907–1913.

\bibitem{quan2023detection}
F.~Quan, X.~Sun, H.~Zhao, Y.~Li, and G.~Qin, ``Detection of rotating stall inception of axial compressors based on deep dilated causal convolutional neural networks,'' \emph{TASE}, vol.~21, pp. 1235--1243, 2023.

\bibitem{appnp}
J.~Gasteiger, A.~Bojchevski, and S.~G{\"u}nnemann, ``Predict then propagate: Graph neural networks meet personalized pagerank,'' \emph{arXiv preprint arXiv:1810.05997}, 2018.

\bibitem{MCS}
R.~E. Tarjan and M.~Yannakakis, ``Simple linear-time algorithms to test chordality of graphs, test acyclicity of hypergraphs, and selectively reduce acyclic hypergraphs,'' \emph{SICOMP}, vol.~13, pp. 566--579, 1984.

\bibitem{zhao2019destination}
Y.~Zhao, K.~Zheng, Y.~Li, H.~Su, J.~Liu, and X.~Zhou, ``Destination-aware task assignment in spatial crowdsourcing: A worker decomposition approach,'' \emph{TKDE}, vol.~32, pp. 2336--2350, 2019.

\bibitem{watkins1992q}
C.~J. Watkins and P.~Dayan, ``Q-learning,'' \emph{ML}, vol.~8, pp. 279--292, 1992.

\bibitem{zheng2022crowdsourced}
L.~Zheng, P.~Cheng, L.~Chen, J.~Yu, X.~Lin, and J.~Yin, ``Crowdsourced fact validation for knowledge bases,'' in \emph{ICDE}, 2022, pp. 938--950.

\bibitem{9474868}
Q.~Tao, Y.~Tong, S.~Li, Y.~Zeng, Z.~Zhou, and K.~Xu, ``A differentially private task planning framework for spatial crowdsourcing,'' in \emph{MDM}, 2021, pp. 9--18.

\bibitem{9101624}
Q.~Tao, Y.~Tong, Z.~Zhou, Y.~Shi, L.~Chen, and K.~Xu, ``Differentially private online task assignment in spatial crowdsourcing: A tree-based approach,'' in \emph{ICDE}, 2020, pp. 517--528.

\bibitem{9101527}
C.~Chai, G.~Li, J.~Fan, and Y.~Luo, ``Crowdsourcing-based data extraction from visualization charts,'' in \emph{ICDE}, 2020, pp. 1814--1817.

\bibitem{8731525}
C.~Chai, J.~Fan, G.~Li, J.~Wang, and Y.~Zheng, ``Crowdsourcing database systems: Overview and challenges,'' in \emph{ICDE}, 2019, pp. 2052--2055.

\bibitem{zhao2023workerchurn}
Y.~Zhao, T.~Lai, Z.~Wang, K.~Chen, H.~Li, and K.~Zheng, ``Worker-churn-based task assignment with context-lstm in spatial crowdsourcing,'' \emph{TKDE}, vol.~35, pp. 9783--9796, 2023.

\bibitem{zhao2021fairness}
Y.~Zhao, K.~Zheng, J.~Guo, B.~Yang, T.~B. Pedersen, and C.~S. Jensen, ``Fairness-aware task assignment in spatial crowdsourcing: Game-theoretic approaches,'' in \emph{ICDE}, 2021, pp. 265--276.

\bibitem{cheng2020real}
Y.~Cheng, B.~Li, X.~Zhou, Y.~Yuan, G.~Wang, and L.~Chen, ``Real-time cross online matching in spatial crowdsourcing,'' in \emph{ICDE}, 2020, pp. 1--12.

\bibitem{deng2024task}
L.~Deng, Y.~Zhao, Y.~Cui, Y.~Xia, J.~Chen, and K.~Zheng, ``Task recommendation in spatial crowdsourcing: A trade-off between diversity and coverage,'' in \emph{ICDE}, 2024, pp. 276--288.

\bibitem{yanwww2022}
Y.~Zhao, X.~Chen, L.~Deng, T.~Kieu, C.~Guo, B.~Yang, K.~Zheng, and C.~S. Jensen, ``Outlier detection for streaming task assignment in crowdsourcing,'' in \emph{{WWW}}, 2022.

\bibitem{liu2022task}
J.~Liu, L.~Deng, H.~Miao, Y.~Zhao, and K.~Zheng, ``Task assignment with federated preference learning in spatial crowdsourcing,'' in \emph{CIKM}, 2022, pp. 1279--1288.

\bibitem{ChenZZ22}
X.~Chen, Y.~Zhao, and K.~Zheng, ``Task publication time recommendation in spatial crowdsourcing,'' in \emph{CIKM}, 2022, pp. 232--241.

\bibitem{lai2022loyalty}
T.~Lai, Y.~Zhao, W.~Qian, and K.~Zheng, ``Loyalty-based task assignment in spatial crowdsourcing,'' in \emph{CIKM}, 2022, pp. 1014--1023.

\bibitem{zhao2023Coalition}
Y.~Zhao, K.~Zheng, Z.~Wang, L.~Deng, B.~Yang, T.~B. Pedersen, C.~S. Jensen, and X.~Zhou, ``Coalition-based task assignment with priority-aware fairness in spatial crowdsourcing,'' \emph{VLDBJ}, 2023.

\bibitem{cheng2019event}
Y.~Cheng, Y.~Yuan, L.~Chen, C.~Giraud-Carrier, G.~Wang, and B.~Li, ``Event-participant and incremental planning over event-based social networks,'' \emph{TKDE}, vol.~33, pp. 474--488, 2019.

\bibitem{lian2020geography}
D.~Lian, Y.~Wu, Y.~Ge, X.~Xie, and E.~Chen, ``Geography-aware sequential location recommendation,'' in \emph{SIGKDD}, 2020, pp. 2009--2019.

\bibitem{wang2019mcne}
H.~Wang, T.~Xu, Q.~Liu, D.~Lian, E.~Chen, D.~Du, H.~Wu, and W.~Su, ``Mcne: An end-to-end framework for learning multiple conditional network representations of social network,'' in \emph{SIGKDD}, 2019, pp. 1064--1072.

\bibitem{li2023competition}
B.~Li, Y.~Cheng, Y.~Yuan, C.~Li, Q.~Jin, and G.~Wang, ``Competition and cooperation: Global task assignment in spatial crowdsourcing,'' \emph{TKDE}, vol.~35, pp. 9998--10\,010, 2023.

\bibitem{ye2021task}
G.~Ye, Y.~Zhao, X.~Chen, and K.~Zheng, ``Task allocation with geographic partition in spatial crowdsourcing,'' in \emph{CIKM}, 2021, pp. 2404--2413.

\bibitem{zhao2023adataskrec}
Y.~Zhao, L.~Deng, and K.~Zheng, ``Adataskrec: An adaptive task recommendation framework in spatial crowdsourcing,'' \emph{TOIS}, vol.~41, pp. 1--32, 2023.

\bibitem{Air-Ground}
Y.~Wang, J.~Wu, X.~Hua, C.~H. Liu, G.~Li, J.~Zhao, Y.~Yuan, and G.~Wang, ``Air-ground spatial crowdsourcing with uav carriers by geometric graph convolutional multi-agent deep reinforcement learning,'' in \emph{ICDE}, 2023, pp. 1790--1802.

\bibitem{Can_You2024}
B.~Rao, X.~Zhang, T.~Zhu, Y.~You, Y.~Li, J.~Duan, Z.~Zhou, and X.~Chen, ``Can you do both? balancing order serving and crowdsensing for ride-hailing vehicles,'' in \emph{IWQoS}, 2024, pp. 1--6.

\bibitem{zhai2019seqst}
D.~Zhai, A.~Liu, S.~Chen, Z.~Li, and X.~Zhang, ``Seqst-resnet: A sequential spatial temporal resnet for task prediction in spatial crowdsourcing,'' in \emph{DASFAA}, 2019, pp. 260--275.

\bibitem{yang2023batch}
Y.~Yang, Y.~Cheng, Y.~Yang, Y.~Yuan, and G.~Wang, ``Batch-based cooperative task assignment in spatial crowdsourcing,'' in \emph{ICDE}, 2023, pp. 1180--1192.

\end{thebibliography}
